\pdfoutput=1

\documentclass{article}

\usepackage{arxiv}

\usepackage[utf8]{inputenc} 
\usepackage[T1]{fontenc}    
\usepackage{hyperref}       
\usepackage{amsfonts}       
\usepackage{nicefrac}       
\usepackage{microtype}      
\usepackage{kpfonts}

\usepackage{chngpage}
\usepackage{graphicx}
\usepackage{subfig}
\usepackage{url}            
\usepackage{booktabs}       
\usepackage{microtype}      
\usepackage{lipsum}
\usepackage{amsmath}
\usepackage{float}
\usepackage{array}
\usepackage{xspace}
\usepackage{amsthm}
\usepackage{multirow}
\usepackage[export]{adjustbox}
\usepackage{stfloats}
\usepackage[sort, numbers]{natbib}
\usepackage{xcolor}
\usepackage{color}
\usepackage{wrapfig}
\usepackage{bbm}
\usepackage{setspace}
\usepackage[font=small,labelfont=bf,tableposition=top]{caption}
\DeclareCaptionLabelFormat{andtable}{#1~#2  \&  \tablename~\thetable}
\definecolor{DarkGreen}{rgb}{0.1,0.5,0.1}
\definecolor{DarkRed}{rgb}{0.5,0.1,0.1}
\definecolor{DarkBlue}{rgb}{0.1,0.1,0.5}

\hypersetup{backref=true,       
    pagebackref=true,               
    hyperindex=true,                
    colorlinks=true,                
    breaklinks=true,                
    urlcolor= black,                
    linkcolor= DarkBlue,                
    bookmarks=true,                 
    bookmarksopen=false,
    filecolor=black,
    citecolor=DarkGreen,
    linkbordercolor=blue
}

\usepackage[font=small,labelfont=bf, labelsep=colon]{caption}
\DeclareCaptionLabelFormat{andtable}{#1~#2  \&  \tablename~\thetable}

\usepackage{boxedminipage}

\newtheorem{theorem}{Theorem}[section]
\newtheorem{definition}[theorem]{Definition}

\newtheorem{failuremode}{Failure mode}

\DeclareMathOperator*{\argmax}{arg\,max}
\DeclareMathOperator*{\argmin}{arg\,min}
\DeclareMathOperator*{\proj}{I-Proj}
\DeclareMathOperator*{\E}{\mathbb{E}}

\DeclareMathOperator*{\sgn}{\mathrm{sgn}}
\newcommand{\Aestk}{\hat A^{\pithetak}}

\newcommand{\DKL}{{D_{\mathrm{KL}}}}
\newcommand{\thetastar}{{\theta^\ast}}
\newcommand{\pithetak}{{\pi_{\theta_k}}}
\newcommand{\pithetastar}{{\pi_{\thetastar}}}
\newcommand{\pithetakplus}{{\pi_{\theta_{k+1}}}}
\newcommand{\piold}{{\pi_{\mathrm{old}}}}

\newcommand{\thetaold}{\theta_{\mathrm{old}}}
\newcommand{\thetanew}{\theta_{\mathrm{new}}}
\newcommand{\MW}{{MW }} 
\newcommand{\cX}{{\mathcal X}}
\newcommand{\cP}{\mathcal{P}}
\newcommand{\cL}{\mathcal{L}}
\newcommand{\LClip}{{\mathcal L}^{\mathrm{CLIP}}}

\newcommand{\LKLorig}{{\mathcal L}^\mathrm{KL, forward}}
\newcommand{\LKLflip}{{\mathcal L}^\mathrm{KL, reverse}}
\newcommand{\grad}{\nabla}
\newcommand{\inner}[2]{\langle #1, #2 \rangle}

\usepackage[affil-it]{authblk} 
\title{Revisiting Design Choices in Proximal Policy Optimization}
\author{
    Chloe Ching-Yun Hsu \quad Celestine Mendler-D\"{u}nner \quad 
     Moritz Hardt 

{\small \{{chloehsu,mendler,hardt}\}@ berkeley.edu}\\ 
 University of California, Berkeley}

\date{}

\begin{document}

\maketitle

\vspace{1.5em}

\begin{abstract}
Proximal Policy Optimization (PPO) is a popular deep policy gradient algorithm.
In standard implementations, PPO regularizes policy updates with clipped probability ratios, and parameterizes policies with either continuous Gaussian distributions or discrete Softmax distributions.
These design choices are widely accepted, and motivated by empirical performance comparisons on MuJoCo and Atari benchmarks.

We revisit these practices outside the regime of current benchmarks, and expose three failure modes of standard PPO. 
We explain why standard design choices are problematic in these cases, and show that alternative choices of surrogate objectives and policy parameterizations can prevent the failure modes.
We hope that our work serves as a reminder that many algorithmic design choices in reinforcement learning are tied to specific simulation environments. We should not implicitly accept these choices as a standard part of a more general algorithm.

\end{abstract}



\section{Introduction}


The PPO algorithm~\cite{schulman2017proximal} is a policy gradient method that is used in diverse high profile reinforcement learning (RL) applications  to train policies, including playing DOTA~\cite{berner2019dota}, manipulating a Rubik's cube~\cite{akkaya2019solving}, and designing chip placement~\cite{mirhoseini2020chip}.

The key feature of the PPO algorithm is a \emph{surrogate objective} for computing policy updates. The surrogate objective regularizes large policy updates, in the spirit of a trust region method, so that each policy update step stays within a close neighborhood around the previous-iteration policy.  This increases the validity of the surrogate objective which is based on data collected from the previous-iteration policy. The algorithm then incrementally refines the policy using multiple steps of stochastic gradient ascent before collecting new data. 

In its full generality PPO refers to a family of algorithms, where the exact choice of the surrogate objective is left as a flexible design choice to the user.
As a consequence, several versions of this algorithm have been proposed in the literature. 
All of them build on the idea of regularizing the distance between the initial and the updated policy, but they differ in their implementation.

A natural approach to incorporate regularization is to use the Kullback-Leibler (KL) divergence between successive policy iterations as a penalty in the surrogate objective.
Such an approach can be theoretically motivated by approximating the Trust Region Policy Optimization~\cite{schulman2015trust} (TRPO) method that is known to monotonically improve policies. 

Rather surprisingly, a different, more ad-hoc surrogate objective has emerged as the most common design choice among practitioners. This alternative surrogate objective is the so-called \textit{clipped objective}. Instead of regularizing the update with a penalty, it greedily ignores any change to the parameter update after the probability ratio between the initial and the updated policy exceeds a predefined threshold. This heuristic is easy to implement and has demonstrated good empirical performance on a number of MuJoCo benchmarks~\cite{schulman2017proximal}. Today, most implementations follow this design choice and it is considered a standard part of PPO -- to the extent that it is often not even mentioned in experimental setups.

The choice of surrogate model is only one example for a standard design choice used to learn policies with PPO.
Another example is the distribution family used for policy parameterization. Standard PPO implementations use diagonal Gaussian distributions on continuous action spaces, and categorical Softmax distributions on discrete action spaces. Alternative families of distributions have been proposed in the literature in the context of other RL algorithms, and there is no reason to rule them out a priori for PPO.

In this paper we revisit these standard design choices for the PPO algorithm. 
Our main concern is that they have been made within the limited regime of current benchmarks, in particular, MuJoCo and Atari. 
It is not a priori clear how robust these methods are when used in different environments. Our goal is to contribute to a more principled understanding of PPO. We point out potential issues with current design choices, investigate alternative design choices, and propose avenues for improving the PPO method.

\subsection{Our contributions}

We design simple test cases to isolate three failure modes where standard PPO (using the clipped objective and Gaussian/Softmax policy parameterization) does not achieve the desired convergence behavior:
On continuous action spaces, standard PPO is unstable when rewards vanish outside bounded support, and it is sensitive to initialization when there are locally optimal actions close to initialization. On discrete action spaces with sparse high rewards, standard PPO often gets stuck at suboptimal actions.
We analyze the reason for these failure modes and explain why they are not exposed by standard benchmarks.

We then revisit alternative design choices and evaluate their performance in these situations. 
On discrete action spaces, we revisit the more principled KL-regularized surrogate objectives, and show that they make PPO more robust to our failure mode example. In fact, KL-regularized PPO even comes with convergence guarantees for one of the settings, and concerningly this favorable property is not preserved by the clipping heuristic.

On continuous action spaces, we find that policy parameterization with Beta distributions is a favorable combination with PPO, because it is more robust to outliers and can be initialized uniformly.
Beta policies avoid the identified failure modes, while also significantly improving PPO performance on MuJoCo environments. 
For example, PPO with beta policy achieves 2x cumulative rewards on the OpenAI Gym~\cite{brockman2016openai} Humanoid-v2 task.

In summary, our study suggests that Beta policy parameterization and KL-regularized objectives should be reconsidered for PPO, especially when moving outside the regime of current benchmarks. 
While we do not claim that our proposed alternatives improve PPO in all settings, we show that the community might have ruled out these variants too soon based on early MuJoCo experiments. 
We hope this paper serves as a reminder that many algorithmic design choices in RL are decided based on specific simulation environments. Without a deeper understanding, we need to reassess these conclusions in new environments.

\section{Background and related work}
\label{sec:background}

We start by providing some background on the PPO algorithm and then introduce the different design choices that we investigate and compare in this manuscript.

\subsection{Proximal Policy Optimization}

PPO is a policy gradient algorithm that learns a parameterized policy $\pi_\theta$. It iteratively updates the parameters $\theta$ of the policy by solving a local optimization problem:
\begin{equation}\thetanew\; \leftarrow \;\argmax_\theta\; \mathcal L(\theta; \hat A^\piold).
\label{eq:obj}
\end{equation}
The surrogate objective $\mathcal L$ uses the advantage estimates $\hat A^{\piold}(s,a)\;\forall s,a$ to assess how much better a particular action $a$ is on state $s$ compared to a randomly sampled action from the previous-iteration policy $\piold(\cdot|s)$. These estimates are obtained by sampling trajectories from $\piold$ prior to each optimization step.
Based on these estimates, the surrogate objective $\mathcal L$ is then optimized to find a new parameter vector $\thetanew$ for the policy $\pi_\theta$.
The optimization \eqref{eq:obj} is not performed exactly in PPO, but using multiple epochs of stochastic gradient ascent.

\begin{figure}[t]
\setlength{\fboxsep}{2mm}
\begin{center}
\small
\color{gray}
\begin{boxedminipage}{\columnwidth}
\color{black}
\vspace{-0.2cm}
\begin{subequations}
\begin{align}
&\LClip\,(\theta) &&:=\E_{a,s\sim\piold}\left[\min\left(\frac{\pi_{\theta}(a|s)}{\piold(a|s)} \;\hat A^{\piold}(a,s),\; \textrm{clip}\left(\frac{\pi_{\theta}(a|s)}{\piold(a|s)}, 1-\epsilon, 1+\epsilon\right)\hat A^\piold(a,s)\right)\right]\label{eq:Lclip}\\
&\LKLorig\,(\theta) && :=\E_{a,s\sim\piold}\left[\frac{\pi_{\theta}(a|s)}{\piold(a|s)} \hat A^{\piold}(a,s)\right] - \beta\DKL(\piold || \pi_{\theta})\label{eq:Lforward}\\
&\LKLflip\,(\theta) && :=\E_{a,s\sim\piold}\left[\frac{\pi_{\theta}(a|s)}{\piold(a|s)} \hat A^{\piold}(a,s)\right] - \beta\DKL(\pi_\theta || \piold)\label{eq:Lreverse}
\end{align}
\end{subequations}
\end{boxedminipage}
\end{center}
\vspace{-0.2cm}
\caption{Design choices for PPO surrogate objective.}
\label{fig:ppo_variants}
\end{figure}

\begin{figure}[t]
\setlength{\fboxsep}{2mm}
\begin{center}
\small
\color{gray}
\begin{boxedminipage}{\columnwidth}
\color{black}
\vspace{-0.2cm}
\begin{subequations}
\begin{align}
&\text{Gaussian}&\pi_\theta(a|s)&:=\mathcal N\left(\mu_\theta(s), \sigma_\theta^2(s)\right)\label{eq:gauss}\\
&\text{Beta}&\pi_\theta(a|s)& := f\left(\frac{a-l}{r-l},\alpha_\theta(s), \beta_\theta(s)\right)\quad \text{ with }\quad f(x,\alpha, \beta):=\frac{\Gamma(\alpha+\beta)}{\Gamma(\alpha)\Gamma(\beta)} x^{\alpha-1}(1-x^{\beta-1})\label{eq:beta}\\
&\text{Softmax}&\pi_\theta(a|s)&:=\frac 1 {c_s} e^{\phi_\theta(s,a)} \quad\text{ with }\quad c_s=\sum_{a'\in\mathcal{A}}e^{\phi_\theta(s,a')}\label{eq:softmax}
\end{align}
\end{subequations}
\end{boxedminipage}
\end{center}
\vspace{-0.2cm}
\caption{Design choices for PPO policy parameterization.}
\label{fig:policy_variants}
\end{figure}

\subsubsection{Policy regularization}

The \emph{surrogate objective} $\cL$ in \eqref{eq:obj} is the key feature of PPO, as it regularizes excessively large policy updates and allows the algorithm to efficiently reuse available data.
There are different variants to implement this regularization, we focus on the three primary variants summarized in Figure~\ref{fig:ppo_variants}.

The first objective $\LClip$ in \eqref{eq:Lclip} corresponds to the clipping heuristic proposed in~\cite{schulman2017proximal}. It  is simple to implement and uses a clipping threshold $\epsilon>0$ to control the size of each policy update. This objective is most popular amongst practitioners because of its good benchmark performance.

The second objective $\LKLorig$  in \eqref{eq:Lforward} uses a soft constraint on the forward KL distance between the initial and the updated policy. The regularization strength is controlled by the regularization parameter $\beta$. This objective was also proposed in \cite{schulman2017proximal} and is closely related to TRPO~\cite{schulman2015trust}, a well understood trust region algorithm related to PPO. 

The third objective $\LKLflip$  in \eqref{eq:Lreverse} uses the reversed KL-distance for regularization. 
This objective has not been studied empirically but has been the main focus of theoretical studies around PPO~\cite{neu2017unified, liu2019neural,geist2019theory}. 
The reason is its close connection to  mirror descent, a well understood optimization method with provable convergence guarantees.

To make the connection between KL-regularized PPO and mirror descent formal, Liu et al.~\cite{liu2019neural} use the representation power of overparameterized neural networks to approximate the infinite-dimensional mirror descent updates, while other works focus on MDPs with finite state and action spaces.
We review and extend these theoretical studies connecting PPO to mirror descent in Appendix~\ref{sec:ppo_connections}. We simplify their exposition and clarify that the assumptions in the convergence analyses of~\cite{liu2019neural,neu2017unified} are satisfied only for certain families of policy parameterizations. These refined theoretical insights help us draw new intuition on the convergence properties of the PPO variants, and how the choice of policy parameterization affects PPO behavior.

\subsubsection{Policy parameterization} 

Standard PPO implementations use diagonal Gaussian distributions for parameterizing the policy $\pi_\theta$ on continuous action spaces, and Softmax distributions on discrete action spaces. In this paper we consider Beta distributions as an alternative parameterization for continuous action spaces. The respective policies $\pi_\theta(a|s)$ are stated in Figure~\ref{fig:policy_variants}.

Beta policy parameterizations have previously been proposed in~\cite{pmlr-v70-chou17a} for the TRPO algorithm. The authors chose Beta policy parameterizations because they can explicitly incorporate action space boundaries $a\in[l,r]$ and eliminate the biased boundary effects caused by truncated Gaussians. In this work we demonstrate that this is not the only benefit of using Beta parameterization, it also leads to more reliable convergence behavior of the PPO algorithm in our test cases and can outperform standard PPO even in settings where boundary effects are not relevant.

\subsection{Empirical studies}
Previous works have already expressed concerns about the robustness of standard PPO. Henderson et al.~\cite{henderson2018deep} highlight the concern that RL algorithm comparisons depend on the environments, and show that PPO and other deep RL algorithms are sensitive to random initialization and reward scaling.
Our work complements previous work with more in-depth analysis of standard PPO's lack of robustness in relation to the two design choices: policy parameterization and surrogate objective.

Recent ablations studies~\cite{engstrom2020implementation,andrychowicz2020matters} investigate how some hyperparameters and design choices affect the performance of PPO and other on-policy RL algorithms on MuJoCo benchmarks.
While these studies provide empirical results on how to optimize PPO performance on MuJoCo benchmarks, our work takes a first step in examining failure modes of standard PPO outside of current benchmarks.
Our work also differs from existing empirical work on the studied design choices. Motivated by the discovered failure modes of standard PPO, we choose to focus on policy parameterizations and surrogate objectives. While the alternative design choices we study have been previously proposed, their impact on PPO performance has been largely unknown and unquestioned. For example, in a recent large-scale study~\cite{andrychowicz2020matters} of more than 50 design choices in on-policy RL algorithms, the choices of clipping as regularization and Gaussian policy are not included in the studied design choices.


\section{Failure modes of standard PPO}
\label{sec:objs}

We start by outlining three failure modes of the common combination of clipped surrogate objective with continuous Gaussian and discrete Softmax policy parameterization. We refer to these design choices as \emph{standard PPO}.

\subsection{Reward signal with bounded support}
\label{sec:non-smooth}

The first failure mode illustrates a scenario where the clipped objective fails to recover from a bad Gaussian policy update step. The clipping mechanism effectively prevents the policy from moving further away once it is outside the trust region, but it does not bound the size of an individual policy update step. This behavior is particularly problematic if a single reward signal can cause the policy to end up in regions with low reward signal. To illustrate this issue, consider the following toy example where the support of the reward signal is bounded:

\begin{figure}[t]
    \subfloat[Reward landscape over action space (left) and heatmap of learned Gaussian policy density across iterations for one training run (right).]{
    \centering
    \includegraphics[width=0.62\linewidth, valign=m]{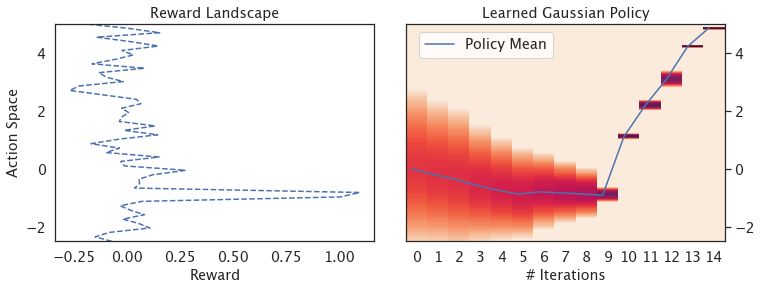}
\label{fig:1dheatmap}
    }
    ~\hspace{1em}~
    \subfloat[Reward of learned Gaussian policies, with mean and std over 20 runs.]{
        \centering
        \includegraphics[width=0.33\linewidth, valign=m]{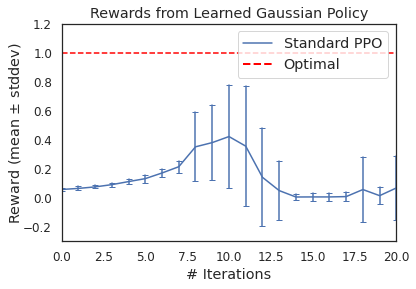}
\label{fig:1dreward}
    }
    \caption{\textit{Failure mode \ref{f1}.} Policy training on the 1-d example outlined in Section~\ref{sec:non-smooth} using standard PPO with Gaussian policy parameterization. 
Initially the policy moves towards the high reward peak, but after 10 iterations the policy drifts away and the reward drops. }
    \label{fig:cts_bandit_failure}
\end{figure}

\begin{failuremode}
\label{f1}
Consider the simple 1-dimensional example on a continuous action space illustrated in Figure~\ref{fig:1dheatmap}. The reward landscape, visualized in the left plot, has a single peak of high reward around $a=-0.9$ and is purely random noise outside the interval $[-1.0, -0.8]$.
For this simple example we illustrate the training of a Gaussian policy with the clipped objective  in the right plot. We see that in the initial phase of training the policy $\pi_\theta$ quickly moves towards the high-reward peak and concentrates by rapidly decreasing the standard deviation. But then, at iteration 10,  it starts to diverge and drifts away into low reward regions. The experienced reward across iterations is shown in Figure~\ref{fig:1dreward}.
\end{failuremode}

The undesirable behavior of Failure mode \ref{f1} can reliably be reproduced in similar 1-dimensional settings, where reward signals are bounded to a subregion of the action space. 
The abruptly vanishing reward signal outside the interval $[-1.0, -0.8]$ is the main culprit. As we will explain it is particularly problematic for the combination of the clipped objective with the Gaussian policy. 
Later, in Section  \ref{sec:fix}, we demonstrate that PPO with the KL-regularized surrogate objective or with alternative policy parameterization can better handle this failure mode example.

\textit{Why is the combination of Gaussian policy parameterization and clipping problematic for this example? }
To better understand why this undesirable behavior occurs, let us inspect the gradient of the clipped objective \eqref{eq:Lclip}. To simplify notation we denote the probability ratio by $r(a|s):=\pi_\theta(a|s)/\piold(a|s)$ and the advantage estimates of action $a$ at state $s$ as $\hat A_{a,s}:=\hat A^\piold(a,s)$.
\begin{equation}
    \nabla_\theta \LClip(\theta) = \E_{a,s\sim\piold} \left[ 
	1 \left\{ \left|r(a|s) - 1\right| < \epsilon 
        \text{ or } \sgn\left(r(a|s) -1\right) \neq \sgn(\hat A_{a,s}) \right\}\;
        r(a|s)   \nabla_\theta \log\pi_\theta(a|s) \hat A_{a,s}
          \right].
    \label{eq:clip_obj_grad}
\end{equation}
The gradient update \eqref{eq:clip_obj_grad} is non-zero as long as the probability ration $r(a|s)$ deviates from $1$ by less than $\epsilon$. As soon as this threshold is exceeded, the clipping mechanism effectively prevents the policy from moving futher away in subsequent iterations. However, it does not regularize individual policy update steps to stay inside the trust region. 
Thus, if update steps are large, a single reward signal can cause the policy to move far away from $\piold$. 
In Failure mode \ref{f1} it is particularly problematic if the policy suddenly jump outside the support of the reward signal. Sampling from such a policy in the next round may contain little signal to recover.

Of course, the strength of the above phenomena critically depend on the choice of policy parameterization. In the following we will explain why a standard Gaussian policy parameterization $\theta = (\mu_\theta, \sigma_\theta^2)$ such as in \eqref{eq:gauss} can be problematic. 

Therefore, recall the gradient update in \eqref{eq:clip_obj_grad}. The advantage value $\hat A_{a,s}$ is weighted by the score function $\nabla_\theta \log\pi_\theta(a|s)$ which measures how quickly the probability density at a given action changes as the distribution parameters change. For Gaussian policies the score function is given by 
\[\nabla_\theta \log\pi_\theta(a|s) =\frac 1 {\sigma_\theta^2(s)} \left[{(a-\mu_\theta(s))}\quad \frac{(a-\mu_\theta(s))^2 - \sigma_\theta(s)^2}{\sigma_\theta(s)}  \right].\]
The expression is large for actions further away from the Gaussian policy mean $\mu$ and grows inversely proportional to the variance of the Gaussian policy. The score function is visualized in Figure~\ref{fig:cts_bandit_failure_explanation} in Section \ref{sec:advKL}. As the policy variance decreases over the course of training, the policy becomes more and more sensitive to actions in the tails. Hence, in later iterations, a single reward signal relatively far from the policy mean can lead to a large policy update and cause the policy to suddenly jump to low reward regions.
Without meaningful reward signals, the mean and the standard deviation of the Gaussian policy fluctuate randomly due to noise.
Once this policy is used to sample new data, the small standard deviation makes it extremely unlikely for PPO to re-discover the high-reward interval because it under-explores the reward landscape. 

\begin{figure}[t]
    \centering
\includegraphics[width=0.95\linewidth,valign=m]{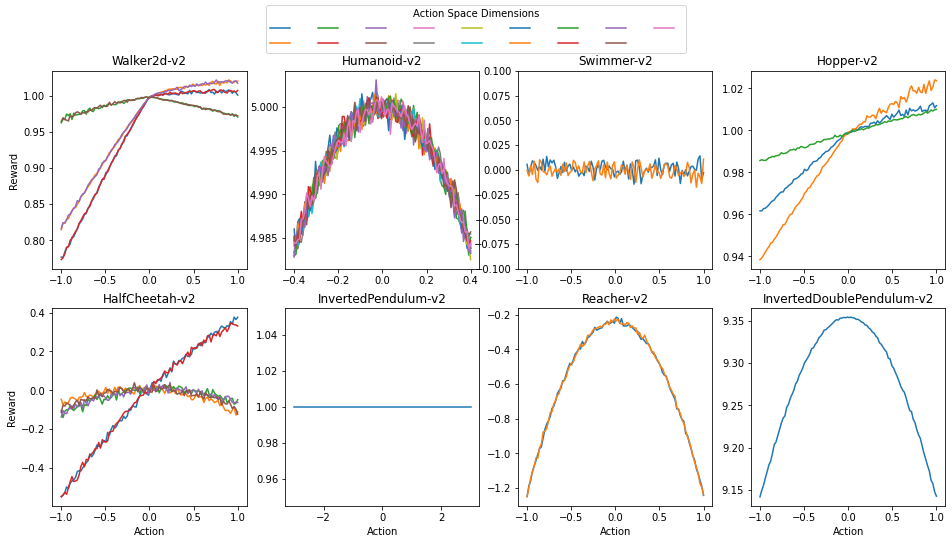}
        \caption{Reward landscapes on initial states of common OpenAI Gym MuJoCo benchmarks. 
}
\label{fig:reward_landscape}

\end{figure}

\textit{Why haven't we observed this failure mode on MuJoCo benchmarks?} This failure mode is hidden for several reasons. First, OpenAI Gym provides designed action space boundaires for each task to facilitate learning and clipps sampled actions outside of the boundaries. Second, most reward landscapes are smooth (Figure~\ref{fig:reward_landscape}), unlike our failure mode example. Third, in many PPO implementations, the variance in the Gaussian policy is a state-independent variable. While state-independent variance avoids the variance to shrink quickly, it also limits the expressive power of the policy, forcing the policy to sample suboptimal actions even on states with higher confidence.

Although failure mode \ref{f1} is only a toy example, reward signals with narrow bounded support are of practical relevance. Such situations can arise due to multiple reasons, such as unknown action boundaries and sparse rewards. If we do not know a priori where meaningful reward signals lie in the action space prior to policy training, we have to set wider action space boundaries that might include vanishing reward signals.
Even with knowledge of action space boundaries, sparse rewards can also result in limited reward signals, such as in robotics, where the robot only receives a high reward for completing a task such as inserting a cable or stacking two blocks~\cite{riedmiller2018learning,vecerik2017leveraging}.

\subsection{High-dimensional discrete action spaces}

\begin{figure}
\begin{adjustwidth}{-.5in}{-.5in}
    \subfloat{
    \centering
    \includegraphics[width=0.4\textwidth]{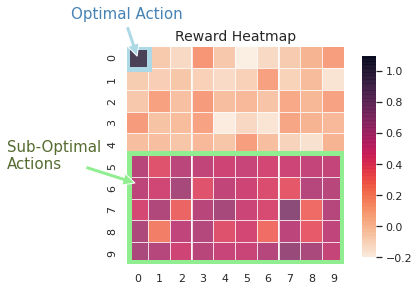}
    }
    \subfloat{
    \centering
    \includegraphics[width=0.65\textwidth]{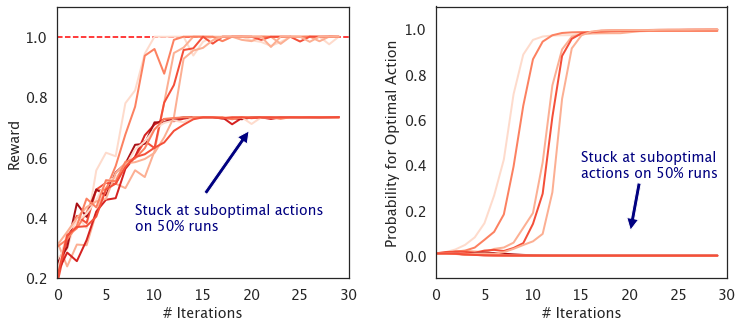}
    }
\end{adjustwidth}
    \caption{ \textit{Failure mode \ref{f3}:} Half of the actions have zero reward, half have suboptimal rewards around 0.5, and only one optimal action has average reward 1. In this discretized environment with 100 actions, {PPO with the clipped objective gets stuck at suboptimal actions around half of the times.}}
    \label{fig:disc_bandit_failure}
\end{figure}

In the second failure mode we illustrate an example where clipping converges to suboptimal actions for distrecte action spaces with the standard softmax policy parameterization. This failure mode is particularly worrying because it happens within the convergence regime of the more principled KL-regularized PPO which clipping aims to approximate. We consider the following setup:

\begin{failuremode}
\label{f3}
Let half of the actions have zero reward, half of the actions have suboptimal reward around 0.5, and there is only one optimal action with average reward 1. This example is  illustrated in Figure~\ref{fig:disc_bandit_failure}. It shows how the clipped objective is likely to converge to suboptimal actions with discrete Softmax policy. For such an environment with 100 actions, standard PPO gets stuck at the suboptimal actions in around 50\% of the runs.
\end{failuremode}

The problem is that when the clipped objective sees only the bad actions (reward 0) and the suboptimal actions (reward 0.5) without seeing the optimal action, it tends to increase the probability ratio of the suboptimal actions by $(1+\epsilon)$, as maximally permitted by the clipping mechanism. After increasing the probability of suboptimal actions in several iterations, the policy is less and less likely to sample the optimal action, and further increases the suboptimal action probabilities.

When testing standard PPO in similar environments with varying number of actions (Figure~\ref{fig:disc_bandit_comp}) from 10 to 100, we see that this failure mode is increasingly problematic as the number of actions increases because the probability of seeing the optimal action in each iteration decreases.

We further note that alternative approaches such as reducing learning rate and increasing batch size can only partially mitigate the issue. For example, when reducing the learning rate by $100\times$ from 0.1 to 0.001, PPO with clipping still fails on 15\% runs when action dimension is 100, while PPO with reverse KL regularization succeeds on all 20 runs. Moreover, reducing the learning rate or increasing the batch size is not an ideal fix as it slows down training and requires more samples. Even though large batch sizes (larger than size of state-action space) can naively fix this failure mode on the simple single-state bandit example, it does not scale to more complex environments. When the state space $\mathcal{S}$ is large, sampling the full action space $\mathcal{A}$ on all states would require a batch size of $|\mathcal{S}| \times |\mathcal{A}|$. Even with relatively large batch sizes in more complex environments, we can not realistically hope to see all actions on each state.

Instead of masking the underlying issue with partial fixes, understanding and analyzing failure modes in a principled way is a necessary first step towards designing robust fixes. In Section~\ref{sec:fix} we show that using KL regularization as an alternative surrogate model provides a scalable and theoretically grounded fix for this failure mode.
It is interesting to see that some alternative algorithm design choices can  be provably more robust, and we hope to encourage the community to keep questioning established algorithms, in particular for new environments.

While the high-dimensionality of action spaces is one aspect of the classical exploration-exploitation tradeoff, existing RL research around exploration mostly focuses on continuous, rather than discrete action space~\cite{osband2016deep, tang2017exploration, plappert2017parameter, osband2019deep}. 
Atari games in OpenAI Gym have 3-18 discrete actions (see Table~\ref{tab:atari_action_dims} in the appendix for the exact statistics). They are unable to reveil this failure mode and assess PPOs performance on higher-dimensional discrete action spaces.

\subsection{Locally optimal actions close to initialization}
\label{sec:suboptimal_failure_mode}

The third example shows that standard PPO can converge to suboptimal actions, even close to initialization. While convergence to suboptimal solutions is a general problem in RL, this example illustrates that for standard Gaussian policy parameterization this is doubly problematic as its suboptimal convergence is additionally aggravated by the initialization. 

\begin{failuremode}
\label{f2}
Consider a simple 1-dimensional example with two reward peaks, as illustrated in Figure~\ref{fig:double_peak_failurea}; one suboptimal peak centered at $+1$ closer to the initial distribution $\mathcal N(0, 1)$, and one optimal peak centered at $-2$ further away from the initial distribution. 
In this example, we see that the policy converges to the suboptimal action around $+1$, as visualized by the heatmap in Figure~\ref{fig:double_peak_failurea}. 
\end{failuremode}

The suboptimal convergence in this example is due to the shape of the initial Gaussian distribution, which is typically set to a diagonal Gaussian policy with zero mean and unit variance in all dimensions. In the initial Gaussian distribution, even though the gradients are of smaller magnitude around the suboptimal peak, they occupy more probability mass and hence cause the aggregate gradient to move towards the suboptimal peak close to $0$. 

\begin{figure}[t]
    \subfloat[Gaussian policy]{
\centering
    \includegraphics[width=0.7\linewidth,valign=m]{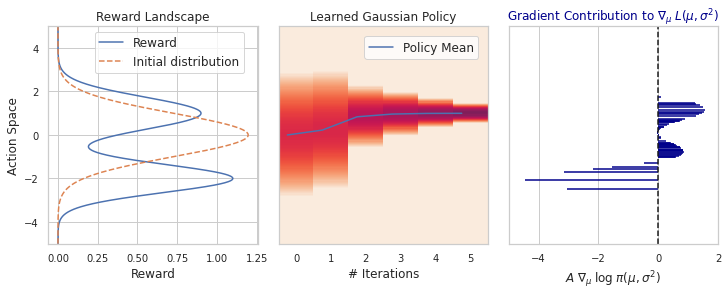}
\label{fig:double_peak_failurea}
}~\hspace{1em}~
    \subfloat[Beta policy]{
    \centering 
    \includegraphics[width=0.23\linewidth,valign=m]{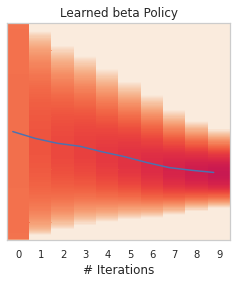}
\label{fig:double_peak_beta}
    }
    \caption{\textit{Failure mode \ref{f2}.} A double peak smooth reward landscape. (a) Suboptimal actions close to initialization are problematic for PPO with clipped objectives and Gaussian policies, as the suboptimal peak occupies more mass in the initial distribution and weighs aggregated gradients towards the suboptimal peak. (b) Beta policy with uniform initialization learns the optimal reward peak.}
    \label{fig:double_peak_failure}
\end{figure}

This failure mode might be alleviated by a careful choice of initialization. This, however, would require a priori knowledge of the reward landscape which is often not reasonable in practice, and if available there would be more effective ways to exploit it. 

In general, this particular issue is not commonly discussed in the literature because it does not arise in the popular OpenAI Gym~\cite{brockman2016openai} implementation of MuJoCo tasks. There, each environment comes with pre-defined action space boundaries and within the action space boundaries most tasks have single-peak reward landscapes, see Figure~\ref{fig:reward_landscape}. 

In Section~\ref{sec:fix} we demonstrate that  Beta policy parameterization can help avoid this undesirable convergence of PPO to suboptimal actions close to initialization, since it can be initialized uniformly in the action space. Of course, there are cases where both Gaussian and Beta policies would converge sub-optimally, not caused by initialization. For example, when there is a very wide action region with relatively high reward and a very narrow action region with the highest reward, both Gaussian and beta policies tend to converge to the suboptimal wide region.


\section{Improvements for PPO failure modes}
\label{sec:fix}
We demonstrate that the KL-regularized surrogate objective and Beta policy parameterizations provide two simple alternatives to avoid the three failure modes of standard PPO discussed in Section \ref{sec:objs}. While both alternative design choices have been previously proposed in the RL literature, they have been largely ignored in combination with the PPO algorithm.

\subsection{KL-regularized surrogate objectives}
We revisit the KL-regularized PPO surrogate objectives \eqref{eq:Lforward} and \eqref{eq:Lreverse} as an alternative to the clipped surrogate objective \eqref{eq:Lclip}, motivated by its close connection to the mirror descent algorithm and the resulting theoretical guarantees.
We show that they make PPO more robust to Failure mode \ref{f1} and \ref{f3}, while having comparable MuJoCo performance.

\subsubsection{Clarification of KL direction}
The KL divergence is asymmetric and it remains unclear from the literature in which direction the KL divergence should be computed for regularizing the PPO surrogate objective. 
There are two possible versions of the KL-regularized objective, \textit{forward-KL} \eqref{eq:Lforward} and \textit{reverse-KL} \eqref{eq:Lreverse}. We are not aware of any empirical comparison of their performance. The forward direction is evaluated in the original PPO paper~\cite{schulman2017proximal} but the reverse direction has solely been studied theoretically.

The KL-divergence is approximately symmetric throughout the algorithm and the KL direction does not significantly affect PPO's performance, due to the regularizing properties of the surrogate objective used in PPO.
To verify, we use the Taylor expansion
$
     r(a|s) \log(r(a|s)) \approx 
     r(a|s) - 1 + \frac{1}{2}(r(a|s) -
     1)^2,
$
and the identity $\int_x \pi_\theta(x) \nabla_\theta \log(\pi_\theta(x)) = 0$, to quantify the difference between the two KL terms:
\begin{equation*}
    \nabla_\theta \DKL(\pi_\theta\;||\;\piold) - \nabla_\theta
    \DKL(\piold\;||\;\pi_\theta) \approx
    \frac{1}{2}\E_{a,s\sim\piold} \left[|r(a|s) - 1|^2 \nabla_\theta(\log \pi_\theta(a|s)) \right].
\end{equation*}
The difference between the two KL-regularized objectives is small in practice, since PPO regularizes the probability ratio $r(a|s)$ to be close to 1, 
Indeed, the correlation between $\DKL(\pi_\theta\;||\;\piold)$ and $ \DKL(\piold\;||\;\pi_\theta)$ is close to 1 throughout the training of common benchmark tasks, as illustrated in Figure~\ref{fig:kl_correlation} in the appendix. As a consequence, PPO with either KL penalty is approximately equivalent to natural policy gradient (NPG)~\cite{kakade2002natural} (explained in Appendix~\ref{sec:ppo_connections}), and in practice often result in similar performance. 

\begin{figure}[t]
    \begin{minipage}[t]{0.47\textwidth}
    \centering
    \vspace{-1em}
    \includegraphics[height=0.65\linewidth, valign=t]{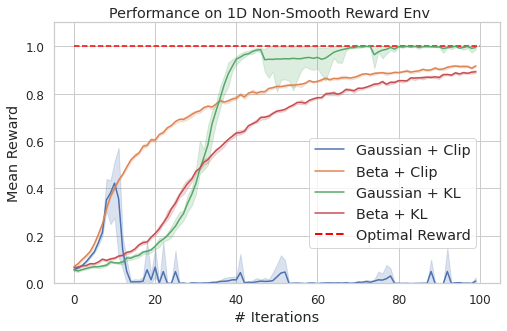}
        \caption{Beta policy and KL regularization both independently avoid Failure mode \ref{f1} (Figure~\ref{fig:cts_bandit_failure}, continuous bandit with bounded reward signals). 
        }
    \label{fig:cts_bandit_comp}
    \end{minipage}\hfill
    \begin{minipage}[t]{0.47\textwidth}
    \centering
    \includegraphics[height=0.65\linewidth,valign=t]{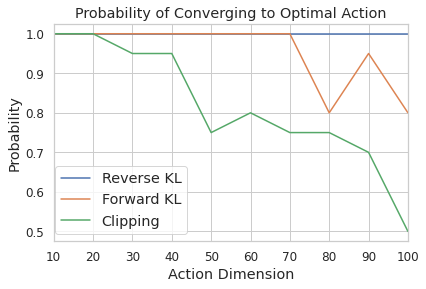}
        \vspace{-0.2em}
        \caption{KL-regularized objectives are more robust in Failure mode \ref{f3} (Figure~\ref{fig:disc_bandit_failure}, high-dimensional discrete action space).
}
    \label{fig:disc_bandit_comp}
    \end{minipage}
\end{figure}

\subsubsection{Advantages of KL-regularization}
\label{sec:advKL}

\textit{Fixing Failure mode \ref{f3}.}
KL-regularized PPO enjoys convergence guarantees when the parameterized policy class is closed under mixture~\cite{liu2019neural}. This includes finite discrete action spaces such as in the example of Failure mode \ref{f3}.
We verify in Figure~\ref{fig:disc_bandit_comp} that KL-regularized PPO indeed converges reliably on this example. 
KL-regularized PPO more effectively avoids suboptimal actions compared to  standard PPO, especially as the number of actions increases. 
When comparing the two KL directions, the theoretically principled reverse-KL penalty (see Appendix~\ref{sec:convergence} for convergence guarantees) has a slight advantage over the forward-KL penalty.

Importantly, the convergence guarantees of PPO for KL-regularized surrogate objectives~\cite{liu2019neural} are not preserved when using the standard clipping heuristic. Failure mode \ref{f3} (Figure~\ref{fig:disc_bandit_failure} and Figure~\ref{fig:disc_bandit_comp}) illustrates one simple example where standard PPO fails even inside the regime of existing theoretical guarantees for KL-regularized PPO. 

The myopic convergence of standard PPO compounds with the sampling effects in exploring high-dimensional action spaces.
To clarify, while KL-regularized PPO always converges optimally as we vary the action dimension from 10 to 100 in Figure~\ref{fig:disc_bandit_comp}, it is not expected that the convergence probability will stay at 1 as we arbitrarily increase the dimensions. This does not violate the theoretical regret bound for KL-regularized PPO, since the regret bound is based on the samples that the policy sees during training. When the action dimension is large compared to the batch size, the optimal action is very unlikely to be sampled, and hence the regret bound does not imply perfect convergence. Empirically, when we further increase the action dimension to 1000, KL-regularized PPO converges to optimal action on around 30\% of the runs, compared to less than 10\% for PPO with clipping.

\textit{Fixing Failure mode \ref{f1}.}
KL-regularized PPO also mitigates Failure mode \ref{f1} on continuous action spaces, as illustrated in Figure~\ref{fig:cts_bandit_comp}. Even though its formal convergence guarantees do not extend to parameterized continuous action spaces with Gaussian policies, it is more effective in  preserving the trust region, compared to clipping. As pointed out in Section \ref{sec:non-smooth}, the clipped objective gets stuck once a large stochastic gradient step is taken and does not provide a reliable mechanism for the policy to come back to the trust region in future steps of the same round. In contrast, a large gradient step with the KL-regularized objective would cause large opposite gradients from the KL penalty, forcing the policy to move back closer to the old policy, before sampling new data.

\textit{MuJoCo performance.}
Empirical results on MuJoCo tasks show that KL-regularized PPO has comparable performance to clipped PPO when in combination with Beta policy. As expected, KL direction does not significantly affect PPO performance, neither for Gaussian nor for Beta policies. The corresponding experimental results can be found in Figure~\ref{fig:kl_direction_comp} and Figure~\ref{fig:clipping_kl_mujoco_comp_beta} in the appendix.

\begin{figure}[t]
    \centering
    \includegraphics[width=0.9\textwidth]{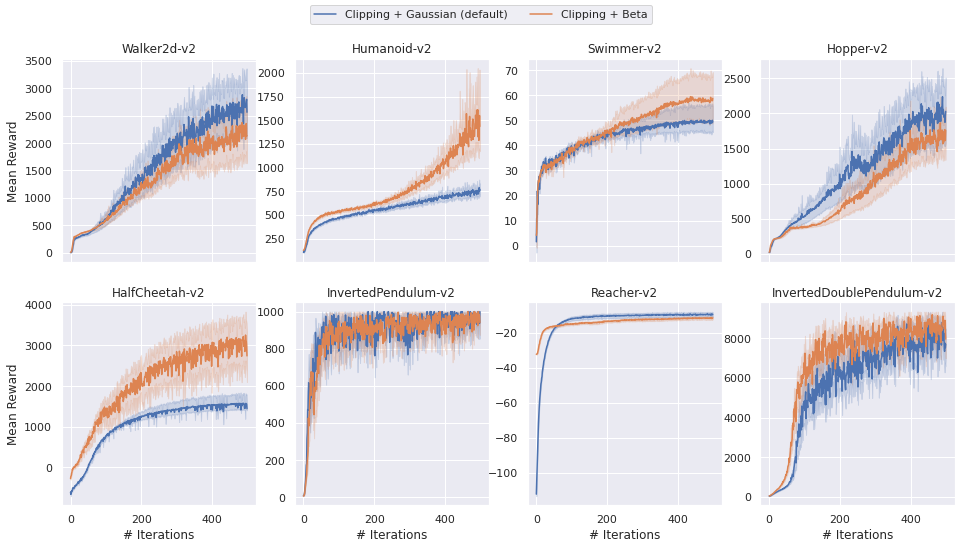}
    \caption{Beta policy significantly improves PPO performance on several MuJoCo tasks compared to Gaussian policy, while maintaining similar performance on other tasks. Most remarkably, Beta policy achieves 2x cumulative reward on Humanoid.}
    \label{fig:beta_gaussian_mujoco_comp_full}
\end{figure}

\subsection{Beta policy parameterization on continuous action spaces}

A second, stand alone fix for the continuous Failure modes \ref{f1} and \ref{f2}, is to use Beta policy parameterizations instead of Gaussians. 
The Beta distribution \eqref{eq:beta} is defined on a bounded interval, and parameterized by two shape parameters $\alpha$ and $\beta$.
Following~\cite{pmlr-v70-chou17a}, we model $\alpha, \beta$ by passing neural network outputs through $x \mapsto \log(1+\exp(x)) + 1$ to ensure $\alpha, \beta \geq 1$.
One might think of the finite support of Beta policies as a restriction, but our evidence in Section~\ref{sec:suboptimal_failure_mode} has shown that Gaussian policies also require prior knowledge of the reward landscape to be initialized properly. As we will see, compared to Gaussian parameterization, Beta policy parameterization utilizes such prior knowledge more effectively.

A number of existing high-profile PPO applications artificially discretize continuous action spaces~\cite{berner2019dota, akkaya2019solving, mirhoseini2020chip} and use Softmax policy parameterization, possibly to avoid the failure modes with Gaussian policies. While discretizing the action space can effectively avoid the failure modes, it largely increases the number of policy parameters, especially in high-dimensional spaces. Beta policies, on the other hand, avoid the failure modes more efficiently without increasing the number of policy parameters. 

\begin{figure}[t]
\label{fig:score}
    \centering
\subfloat[Gaussian policy parameterization]{
\label{fig:cts_bandit_failure_explanation}
\includegraphics[width=0.9\linewidth]{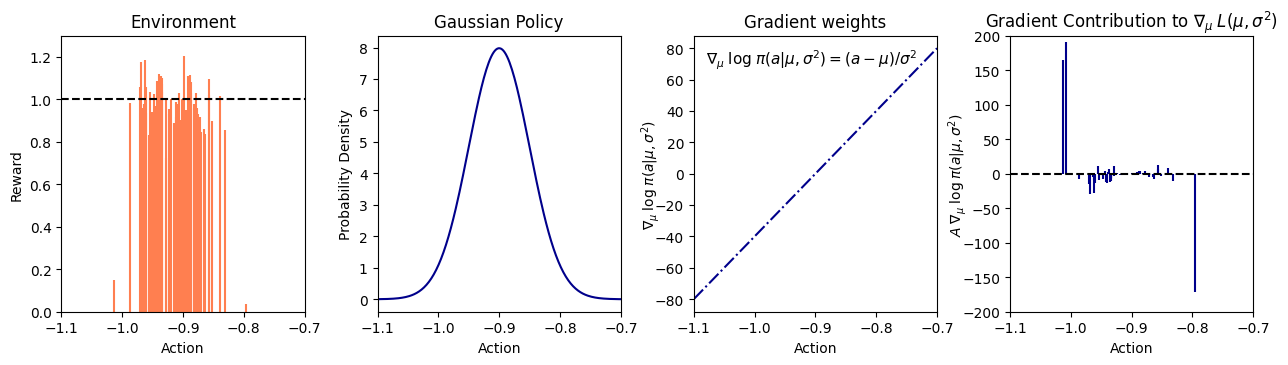}}\\
\subfloat[Beta policy parameterization]{
    \label{fig:cts_bandit_failure_explanation_beta}
    \includegraphics[width=0.9\linewidth]{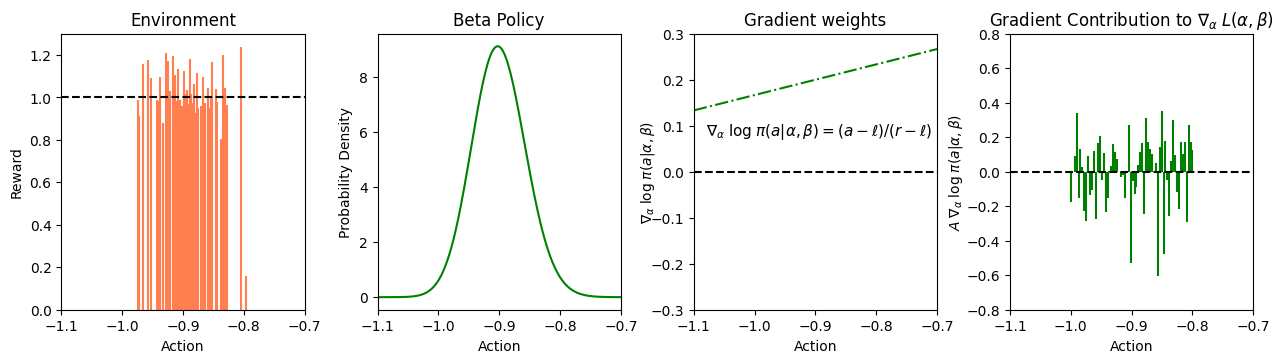}}
    \caption{A single PPO iteration in Failure mode \ref{f1} (Figure~\ref{fig:cts_bandit_failure}) for Gaussian policies (a) and Beta policies (b), with wide action boundary $[-3, 0]$ for fair comparison. We show: Sampled actions and their rewards (left), policy action distribution (middle left), weighting function of action samples in the aggregated gradients (middle right), and each action sample's contribution to the aggregated gradient (right). While the beta policy illustrated here has approximately the same shape as the Gaussian policy, the beta parameterization leads to more even weighting of the gradients across the actions sampled, thus preventing excessively large gradient updates caused by a few actions in the tail.}
\end{figure}

\textit{Fixing failure mode \ref{f1}.}
In the first failure mode (Section~\ref{sec:non-smooth}), the key problem is that Gaussian policies excessively upweigh actions in the tails, especially when the policy standard deviation is small.
For Beta policy parameterization the likelihood of given actions in the tails do not change as rapidly and hence it weighs action samples more evenly, see Figure~\ref{fig:cts_bandit_failure_explanation_beta} in comparison to Figure~\ref{fig:cts_bandit_failure_explanation}. This effectively avoids the first failure mode as demonstrated in Figure \ref{fig:cts_bandit_comp}.

Note that both, the KL-regularized surrogate objective discussed in Section \ref{sec:advKL}, and Beta policy parameterization, are stand-alone fixes for failure mode \ref{f1} through different mechanisms. While the KL-regularized surrogate objective does not change the steep weighting of gradients by the Gaussian score function, it more effectively preserves the trust region, compared to clipping and thereby reduceds the effect of large policy updates caused by actions in the tails.

To eliminate boundary effects, we intentionally set the action boundaries to be broader than the non-zero reward region. Hence, the gain of Beta policy observed in our example is orthogonal to the advantages of Beta policy related to boundary effects pointed out in~\cite{pmlr-v70-chou17a}.

\textit{Fixing failure mode \ref{f2}.}
The culprit of Failure mode \ref{f2} (discussed in Section~\ref{sec:suboptimal_failure_mode}) is the large probability mass occupied by the suboptimal reward peak in the initial distribution. A natural fix is to initialize the policy more uniformly. Uniform initialization is difficult for Gaussian policies, as it requires a large variance which in turn puts large probability mass outside the action boundaries. 
With Beta policies, when $\alpha = \beta = 1$, the Beta distribution is by definition uniform across the interval. With the softplus parameterization of $\alpha = \log(1 + \exp(x_{\alpha})) + 1$, $\alpha=1$ corresponds to $x_{\alpha} = -\infty$, and similarly for $\beta$. In practice, we initialize with $x_{\alpha}, x_{\beta} \approx -4$ and hence $\alpha, \beta \approx 1.018$ for a close approximation of the uniform distribution. This eliminates suboptimal convergence due to bad initialization and hence fixes the failure mode on multi-peak reward landscapes, as shown in Figure~\ref{fig:double_peak_beta}.

\textit{MuJoCo performance.} In addition to preventing the two above failure modes, we found that Beta policy parameterization also improves benchmark performance on some MuJoCo tasks, compared to Gaussian policies, as illustrated in Figure~\ref{fig:beta_gaussian_mujoco_comp_full}. In particular, PPO achieves 2x cumulative reward on the Humanoid-v2 task (1M steps) with beta policy, compared to Gaussian policy.

\section{Discussion}

Many algorithmic design choices in RL are made based on empirical comparisons and it is easy to forget that their justification is limited to the regime of the chosen benchmark tasks. 
Our study highlights, at the example of PPO, that widely accepted design choices do not necessarily generalize to new environments.
We conduct failure mode analyses on synthetic test cases in order to understand and diagnose the broader implications of such design choices. 
We believe that such an approach offers a principled tool to isolate specific convergence issues of an algorithm and it augments classical benchmarks that often confound many different aspects of the environment. 

We emphasize two main insights of our PPO study: First, \textit{Beta policy parameterization} has attractive properties compared to standard Gaussian parameterization. i) It avoids excessively large stochastic gradient updates caused by actions in the distribution tails, which are particularly problematic in combination with the clipped surrogate objective. ii) It allows for approximately uniform initialization and is thus less sensitive to suboptimal initialization. iii) As pointed out by prior work~\cite{pmlr-v70-chou17a} it also eliminates the bias towards boundaries in truncated Gaussians on bounded action spaces.
Second, the advantage of clipping over the more principled \textit{KL-regularized objective} is limited to the regime of MuJoCo benchmarks with Gaussian policies. On synthetic examples, KL-regularized objective is more robust. 
On MuJoCo benchmarks with Beta policy, KL-regularized PPO has similar or better performance than clipped PPO in most tasks.

While our main analysis focuses on the choice of surrogate objective and policy parameterization, we include several surprising findings around two other common implementation choices in PPO in Appendix~\ref{sec:appendix_scaling} -- advantage normalization and reward scaling.
Briefly, we find that constant reward scaling can replace the complex reward scaling scheme studied in~\cite{engstrom2020implementation} without loss in performance; that advantage normalization helps PPO on some tasks but hurts on others; and that advantage normalization effectively anneals the mirror descent step size.
We hope that our initial findings provide a valuable starting point for further work towards fully understanding these design choices.

Finally, we believe that our insights on PPO can also help to better understand other policy gradient algorithms such as Natural Policy Gradient (NPG)~\cite{kakade2002natural} and V-MPO~\cite{song2019v}, as they are closely connected to KL-regularized PPO as discussed in Appendix~\ref{sec:ppo_connections}.

\section*{Achnowledgements}

The authors would like to thank Ben Recht for insightful discussions and for providing feedback on the initial draft of this work. 
In addition, the second author wishes to acknowledge support from the Swiss National Science Foundation Early Postdoc Mobility
Fellowship Program.

\section*{Code availability}
Our code is available at \url{https://github.com/chloechsu/revisiting-ppo}.

\newpage

\bibliographystyle{plain}
\bibliography{references}

\newpage
\section*{Appendix}
\appendix
\section{KL-regularized PPO: convergence guarantees and connections to mirror descent, NPG, and V-MPO}
\label{sec:ppo_connections}

We first focus on the connection between KL-regularized PPO and mirror descent. We extend and simplify existing expositions of this connection and derive convergence guarantees for PPO with parameterized policy classes.

Then, we also connect KL-regularized PPO to natural policy gradient (NPG) and V-MPO.
We note that PPO with KL penalty in either direction is approximately equivalent to both Natural Policy Gradient~\cite{kakade2002natural} and mirror descent with I-projection, while V-MPO~\cite{song2019v} is equivalent to mirror descent with reverse I-projection (also known as M-projection).

\subsection{Connection to mirror descent}
\label{sec:mirror}

Mirror Descent has proven a valuable theoretical tool for deriving convergence results for policy gradient-type methods~\cite{liu2019neural,agarwal2019optimality,neu2017unified,geist2019theory}. 

We recall that mirror descent~\cite{beck2003mirror} 
in its general form is applied to an objective $f$ and optimizes over a distribution family $\mathcal{P}$. Each mirror descent iteration optimizes for a linear approximation of $f$ plus a KL-regularization term:
\begin{equation}
\pi^{k+1} := \argmin_{\pi \in \mathcal{P}} \left\{\inner{\grad f(\pi^k)}{\pi} + \frac{1}{\eta} \DKL(\pi\;||\;\pi^k)\right\},
\end{equation}
where $\eta$ denotes the step size and the negative entropy is used for the Bregman divergence.

When applied to the RL setting where the goal is to maximize expected advantage\footnote{The advantage $A^{\pi}(s,a)$ is a proxy for how much better a particular action $a$ is on state $s$ compared to a randomly sampled action from $\pi(\cdot|s)$.} we choose $f$ to measure the expected advantage at each state $s$, i.e.,
\[f_s(\pi_\theta) = -\E_{a\sim\pi_\theta(\cdot|s)}[A^{\pithetak}(s, a)]\;\forall s.\] 
The corresponding infinite-dimensional gradient is $\grad f_s(\pithetak) = A^{\pithetak}(s, \cdot)$. Hence, one step of  mirror descent for this particular choice of $f$ is equivalent to maximizing the $\LKLflip(\theta)$ PPO surrogate objective in \eqref{eq:Lreverse}.
\begin{equation}
\pithetakplus(\cdot|s) := \argmax_{\theta \in \Theta} \left\{ \E_{a\sim\pi_\theta(\cdot|s)}[A^{\pithetak}(s, a)] - \frac{1}{\eta} \DKL(\pi_\theta\;||\; \pithetak)\right\}.
    \label{eq:mirror_descent_rl}
\end{equation}

\subsection{Connection to projected multiplicytive weights}

Mirror descent can be shown to follow the multiplicative weights (MW) update rule. 
Multiplicative weights (MW) minimize regret in an online learning setting with respect to payoff functions $m_k: \cX \rightarrow [-\rho, \rho]$ with width $\rho$. Given a domain $\cX$, \MW starts with the uniform distribution $\pi_0$. In the k-th round, \MW observes payoffs $m_k(x)$ for each $x\in\cX$ and updates the distribution $\pi_k$ according to the following update rule:
\begin{equation}\label{eq:mw_finite_update}
    \pi_{k+1}(x) \propto \pi_k(x)\; e^{\eta m_k(x)},
\end{equation}
where $\eta \in (0, \frac{1}{2})$ is the \MW learning rate.
Intuitively, the update rule boosts elements with high payoff and down-weigh elements with low payoff\footnote{The update rule is sometimes written in multiplicative weights literature with penalty instead of payoff, and with a linear factor $1+\eta m_k$ instead of the exponential factor $e^{\eta m_k}$. The differences lead to slightly different bounds. See Section~2 in \cite{arora2012multiplicative} for a detailed discussion.}.

\textit{Information projection.}
The  \MW update  is nonparameteric and hence the updated policy $\pi^{MW}$ proposed by \MW is not necessarily in the parameteric family $\cP$.
When working with parameterized distributions, information projection is a natural way to project the exact \MW update to a parameterized family $\cP = \{\pi_\theta: \theta \in \Theta\}$.
Information projection finds the closest distribution within the family, as measured by KL-distance.
\begin{definition}[I-projection]
The information projection of a probability distribution $q$ onto a set of distributions $\cP$ is
    \begin{equation}\label{eq:proj_def}
        \proj_{\cP}q = \argmin_{p\in\mathcal{P}}\DKL(p\;||\;q).
    \end{equation}
\end{definition}

Therefore, reverse-KL-regularzied PPO also follows projected MW, with target updates proportional to exponentiated advantages:
    \begin{equation}
        \label{eq:proj_mw_rule}
        \pi^{MW}(\cdot|s) \propto \pithetak(\cdot|s)\; \exp(\eta A^{\pithetak}(s,\cdot));\;\;\;\;
        \pithetakplus(\cdot|s) = \proj_{\{\pi_\theta(\cdot|s):\; \theta \in \Theta\}}\pi^{MW},
    \end{equation}
With information projection, we can project the infinite-dimensional target \MW update to a parameterized family. In practice, \MW with information projection formulates an optimization objective, requiring only an implicit representation of the target \MW update.

Using the definition of KL-divergence we can write
\begin{align*}
    \DKL\Big(\pi_\theta \;||\; \frac{1}{Z_k}\pi_k\; e^{\eta m_k}\Big)
    &= \log Z_k + \int \pi_\theta (\log \pi_\theta - \log\pi_k - \eta m_k) \\
   & = \log Z_k + \DKL(\pi_\theta\;||\;\pi_k)- \eta \E_{\pi_k}\Big[\frac{\pi_\theta}{\pi_k} m_k\Big].
    \end{align*}
Hence, for $\beta = 1/\eta$  minimizing the information projection distance is equivalent to maximizing the KL-regularized PPO objective $\LKLflip$ \eqref{eq:Lreverse}.

\subsection{Convergence guarantees}\label{sec:convergence}

Building on the connection of KL-regularized PPO to mirror descent and multiplicative weights, we review and extend existing convergence analysis.

The KL-regularized version of PPO inherits convergence guarantees of mirror descent for policy families that are closed under mixture. Thius includes softmax parameterization of discrete action spaces. In the following we will extend existing convergence analyses to cover a more practical setting of PPO, including approximate updates and adaptive step sizes.

Therefore, we extend existing \MW convergence guarantees to approximately projected \MW on closed and convex distribution families.
A parameterized family $\cP=\{\pi_\theta:\theta\in\Theta\}$ is convex when $\cP$ is closed under mixture.
Examples of convex families include distributions with a bounded support and distributions with a bounded probability density range.
In contrast to \MW, each update step in PPO \eqref{eq:proj_mw_rule} is only solved approximately via a fixed number of stochastic gradient steps.
To capture the approximation, we introduce the notion of approximate projection.
\begin{definition}[Approximate I-projection]
A distribution $\tilde p \in \cP$ is an $\alpha$-approximate information projection of $q$ if for any $p \in \cP$, \begin{equation*}
    \DKL(p\;||\;\tilde p) \leq \DKL(p\;||\;\proj_{\cP}q) + \alpha.
\end{equation*}
\end{definition}
Intuitively, $\tilde p$ is an $\alpha$-approximate projection of $q$ if it is close to the exact projection $\proj_{\cP}q$, such that the KL-distance from any other distribution in the family to $\tilde p$ is about the same as its KL-distance to the exact projection $\proj_{\cP}q$.

When mirror descent does not solve for the $\argmax$ in (\ref{eq:mirror_descent_rl}) exactly, it corresponds to an approximate information projection of the multiplicate weights update rule.
An \textit{$\alpha$-approximate projected \MW update} is an $\alpha$-approximate projection of the implicitly represented infinite-dimensional \MW update distribution from the update rule~\eqref{eq:mw_finite_update}.

We also review Bregman's theorem here, which will be used in the convergence analysis.
\begin{theorem}[Bregman]
    \label{thm:bregman}
Let $p, q$ be two distributions such that $p$ is in the non-empty closed convex set $\Gamma$ of measures. Then,
\begin{equation*}
    \DKL(p\;||\;\proj_{\Gamma}q) + \DKL(\proj_{\Gamma}q \;||\;q) \leq \DKL(p\;||\;q).
\end{equation*}
\end{theorem}

Now we are ready to state the convergence guarantees for projected multiplicative weights. We first state Theorem~\ref{thm:proj_mw_bound} in projected multiplicative weights language, and restate the theorem as Theorem~\ref{thm:ppo_bounds} in RL language.
The proof extends main ideas from the proof of Lemma 4.1 in~\cite{barak2009uniform} with a finer-grain analysis, using Bregman's theorem.

\begin{theorem}
\label{thm:proj_mw_bound}
    Let $\cP = \{\pi_\theta: \theta \in \Theta\}$ be a family of distributions closed under mixture. Let $m_k: \cX \rightarrow [-\rho_k, \rho_k]$ be payoff functions. Starting from any initial distribution $\pi_{\theta_0}$, after $K$ rounds of $\alpha$-approximate projected multiplicative weights update with step size $\eta_1, \cdots, \eta_K$ such that $\eta_k \in (0,\frac{1}{\rho_k})$, the payoff difference between $\pi_{\theta^k}$ and the optimal distribution $\pi_{\thetastar}$ is bounded by
\begin{equation*}
    \sum_{k=1}^K\eta_k\left(\E_{\pi_{\thetastar}}[m_k] - \E_{\pi_{\theta_k}}[m_k]\right) \leq
    \DKL(\pi_\thetastar\;||\;\pi_{\theta_0}) + \alpha K 
    + \sum_{k=1}^K \eta_k^2 \E_{\pi_{\theta_k}}[m_k^2],
\end{equation*}
and in the special case of constant step size $\eta$ and constant payoff function width $\rho$, we can simplify the bound as
\begin{equation}
    \frac{1}{K}\sum_{k=1}^K \left(\E_{x\sim\pi_\thetastar}[m_k(x)] -  \E_{x\sim\pi_{\theta_k}}[m_k(x)]\right)
    \leq \eta \rho^2 + \frac{\alpha}{\eta} + \frac{1}{\eta K}\DKL(\pi_{\thetastar}\;||\;\pi_0).
\end{equation}
\end{theorem}

\begin{proof}
Let $p^{MW}_k$ be the multiplicative weights update $p^{MW}_k = \pi_{\theta_k}e^{\eta m_k}/Z_k$. For any $\theta$, by definition,
\begin{equation}\label{eq:proof_kl_diff}
    \DKL(\pi_\theta\;||\;p^{MW}_k) - \DKL(\pi_\theta\;||\;\pi_{\theta_k}) \\
    = -\int \pi_\theta \log \frac{p^{MW}_k}{\pi_{\theta_k}}
    = \log Z_k - \eta \E_{\pi_\theta}[m_k].
\end{equation}
    Meanwhile, since $m_k$ is bounded by $[-\rho_k, \rho_k]$, $\eta_k m_k$ is bounded by $[-1, 1]$, so using $e^x \leq 1 + x + x^2$ for $x\in[-1,1]$,
\begin{equation*}
    Z_k = \E_{\pi_{\theta_k}}[e^{\eta_k m_k}] \leq 1 + \eta_k\E_{\pi_{\theta_k}}[m_k] + \eta_k^2\E_{\pi_{\theta_k}}[m_k^2],
\end{equation*}
and using $log(1+x) \leq x$,
\begin{equation}
    \log Z_k \leq \eta_k\E_{\pi_{\theta_k}}[m_k] + \eta_k^2\E_{\pi_{\theta_k}}[m_k^2].
\end{equation}
Substituting into Equation~\ref{eq:proof_kl_diff}, we have
\begin{equation}\label{eq:proof_kl_diff_sub}
    \DKL\left(\pi_\theta\;||\;p^{MW}_k\right) - \DKL(\pi_\theta\;||\;\pi_{\theta_k}) \\
    \leq \eta_k\left(\E_{\pi_{\theta_k}}[m_k] - \E_{\pi_{\theta}}[m_k]\right) + \eta_k^2\E_{\pi_{\theta_k}}[m_k^2]
\end{equation}
By Bregman's Theorem,
\begin{equation}\label{eq:proof_bregman_thm}
    \DKL\left(\pi_\theta\;||\;p^{MW}_k\right) \geq \DKL\left(\pi_\theta\;||\;\proj_{\cP}p^{MW}_k\right).
\end{equation}
Since $\pi_{\theta_{k+1}}$ is an $\alpha$-approximation of $\proj_{\cP}p^{MW}_k$,
\begin{equation}\label{eq:proof_approx_proj}
\DKL\left(\pi_\theta\;||\;\proj_{\cP}p^{MW}_k\right) \geq \DKL(\pi_\theta\;||\;\pi_{\theta_{k+1}}) - \alpha.
\end{equation}
Therefore, combining Equation~\ref{eq:proof_kl_diff_sub}, Equation~\ref{eq:proof_bregman_thm}, and Equation~\ref{eq:proof_approx_proj},
\begin{equation}\label{eq:proof_kl_diff_one_step}
    \DKL(\pi_\theta\;||\;\pi_{\theta_{k+1}}) -\DKL(\pi_\theta\;||\;\pi_{\theta_k}) \\
    \leq \eta_k\left(\E_{\pi_{\theta_k}}[m_k] - \E_{\pi_{\theta}}[m_k]\right) + \eta_k^2\E_{\pi_{\theta_k}}[m_k^2] + \alpha.
\end{equation}
From the telescope sum of Equation~\ref{eq:proof_kl_diff_one_step} from $k=1$ to $K$, we have
\begin{equation*}
    \DKL(\pi_\theta\;||\;\pi_{\theta_K}) -\DKL(\pi_\theta\;||\;\pi_{\theta_0}) \\
    \leq \alpha K + \left(\sum_{k=1}^K\eta_k\left(\E_{\pi_{\theta_k}}[m_k] - \E_{\pi_{\theta}}[m_k]\right)\right) + \sum_{k=1}^K \eta_k^2 \E_{\pi_{\theta_k}}[m_k^2],
\end{equation*}
and hence
\begin{equation*}
    \sum_{k=1}^K\eta_k\left(\E_{\pi_{\theta}}[m_k] - \E_{\pi_{\theta_k}}[m_k]\right) \leq
    \DKL(\pi_\theta\;||\;\pi_{\theta_0}) + \alpha K
    + \sum_{k=1}^K \eta_k^2 \E_{\pi_{\theta_k}}[m_k^2].
\end{equation*}
In the special case of constant step size $\eta$, we get the desired bound
\begin{equation*}
    \frac{1}{K}\sum_{k=1}^K\E_{\pi_{\theta}}[m_k] - \frac{1}{K}\sum_{k=1}^K\E_{\pi_{\theta_k}}[m_k] \leq \frac{\eta}{K}\sum_{k=1}^K\E_{\pi_{\theta_k}}[m_k^2] + \frac{\alpha}{\eta} +\frac{1}{\eta K}\DKL(\pi_\theta\;||\;\pi_{\theta_0}).
\end{equation*}

If we further substituting with $\E_{\pi_{\theta_k}}[m_k^2]\leq \rho^2$ for constant payoff width $\rho$, we get the simpler but coarser bound.
\end{proof}

The $\frac{\alpha}{\eta}$ term captures the additional error caused by PPO not exactly optimizing the surrogate objective in each iteration.
With $\alpha = 0$ and $\pi_0$ being the uniform distribution, Theorem~\ref{thm:proj_mw_bound} can recover the regret bound for exact \MW on finite domain $\cX$ (see Theorem 2.3 in~\cite{arora2012multiplicative}), using $\DKL(\pi\;||\;\pi_0) \leq \ln|\cX|$ on discrete spaces. 

\begin{theorem}[Restating Theorem~\ref{thm:proj_mw_bound}]
    \label{thm:ppo_bounds}
    Let $\cP = \{\pi_\theta: \theta \in \Theta\}$ be a family of policies closed under mixtur. Assume bounded advantages $\Aestk \in [-\rho_k, \rho_k]$, and assume the gradient steps in each PPO iteration achieve an $\alpha$-approximate projection of the surrogate objective \eqref{eq:Lreverse} on state $s$. Starting from initial policy $\pi_{\theta_0}$, after $K$ iterations with step sizes $\eta_1, \cdots, \eta_K$ such that $\eta_k \in (0,\frac{1}{\rho_k})$, we can bound the difference in advantages between $\pithetak$ and the optimal policy $\pi_{\thetastar}$ on state $s$ by
\begin{multline*}
    \sum_{k=1}^K\eta_k \left(\E_{a \sim \pithetastar(\cdot|s)}[\Aestk(s,a)] - \E_{a\sim\pithetak(\cdot|s)}[\Aestk(s,a)]\right) \leq \\
    \DKL(\pithetastar(\cdot|s)\;||\;\pi_{\theta_0}(\cdot|s)) + \alpha K 
    + \sum_{k=1}^K \eta_k^2 \E_{a\sim\pithetak(\cdot|s)}[\Aestk(s,a)^2],
\end{multline*}
and in the special case of constant step size $\eta$ and constant advantage width $\rho$ on discrete action spaces, we can simplify the bound as
    \begin{equation*}
        \frac{1}{K}\sum_{k=1}^K \left(\max_{a}\Aestk(s,a) -  \E_{a\sim\pithetak(\cdot|s)}[\Aestk(s,a)]\right) 
    \leq \eta \rho^2 + \frac{\alpha}{\eta} + \frac{1}{\eta K}\DKL(\pithetastar(\cdot|s)\;||\;\pi_0(\cdot|s)),
\end{equation*}
meaning the average action under $\pithetak$ is close to the best action under advantage function $\Aestk$.
\end{theorem}

On discrete action spaces, when using uniform initialization, the upper bound in Theorem~\ref{thm:ppo_bounds} holds with $\frac{1}{\eta K}D_{\mathrm{KL}}(\pi_{\theta^*}||\pi_{\theta_0}) \leq \log|\mathcal{A}|$. As the number of training iterations $K$ increases, $\frac{1}{\eta K}D_{\mathrm{KL}}(\pi_{\theta^*}||\pi_{\theta_0})$ vanishes, and the impact of initialization diminishes.

\subsection{Connection to Natural Policy Gradient}
The Hessian of the KL penalty term in either KL-direction is Fisher information matrix $F(\theta_k)$, and the second-order Taylor expansion is 
\[\DKL(\pi_\theta\;||\;\pithetak) \approx \DKL(\pithetak\;||\;\pi_\theta) \approx \frac{1}{2}(\theta-\theta_k)^T F(\thetaold) (\theta-\theta_k).\] 
Therefore, optimizing either KL-regularized objective \eqref{eq:Lforward} or \eqref{eq:Lreverse} results in 
\[\theta_{k+1} \approx \theta_k -\frac{1}{\beta} F^{-1}(\theta_k) \nabla_\theta L(\theta_k),\] 
which is the natural gradient update of the unregularized policy gradient objective.

Each iteration of KL-regularized PPO is approximately equivalent to one natural gradient update in general. As a special case, on finite Markov Decision Processes (MDPs) with finite state space and softmax parameterization for finite action space, the Fisher information matrix is constant, and both NPG and KL-regularized PPO correspond to exact multiplicative weights.

\subsection{Comparison with V-MPO}
Compared to KL-regularized PPO as multiplicative weights with information projection (\ref{eq:proj_mw_rule}), V-MPO corresponds to multiplicative weights with moment projection. Information projection minimizes KL distance $\DKL(\pi_\theta\;||\;p_{\MW})$ to the target nonparametric distribution, where as moment project minimizes the flipped KL distance $\DKL(p_{\MW}\;||\;\pi_\theta)$.

On-Policy Maximum A Posteriori Policy Optimization (V-MPO)~\cite{song2019v} is also related to multiplicative weights.
The E-step in V-MPO constructs the nonparametric target distribution $\psi$ by using the same multiplicative weight updates rule as in (\ref{eq:proj_mw_rule}), and the M-step projects the target distribution to the parameterized family by maximizing weighted maximum likelihood loss $\sum_{s,a} \psi(s,a)\log \pi_\theta(a|s)$. Maximizing this weighted log likelihood is equivalent to minimizing $\DKL(\psi\;||\;\pi_\theta)$, known as moment projection.
V-MPO only takes samples corresponding to the top half advantages in the weighted maximum likelihood loss, which can be interpreted in the multiplicative weights framework as setting the bottom half advantages to $-\inf$.

While V-MPO and PPO are related in the multiplicative weights update form, information projection and moment projection results in different projections: when the nonparametric target distribution is multi-modal cannot be fit within the parametric family, information projection is more mode-seeking, whereas moment projection tends to be spread out more to cover the entire support.

\newpage
\section{Derivation of surrogate objective gradients}
\label{sec:gradients_derivation}

We will use $\mathbbm{1}$ to represent the indicator function and write $r(a|s):=\frac{\pi_{\theta}(a|s)}{\piold(a|s)}$ and $\hat A_{a,s}:= \hat A^{\piold}(a,s)$ to simplify notation.

\subsection{Unregularized objective}
The gradient of the unregularized objective $\cL(\theta)= \E_{a,s\sim \piold}\left[r(a|s) \hat A_{a,s}\right]$ corresponds to
\begin{equation*}
    \nabla_\theta \cL(\theta) 
    = \E_{a,s\sim\piold}\left[
        r(a|s)\hat A_{a,s} \nabla_\theta\log\pi_\theta(a|s) \right]
\end{equation*}

\subsection{Clipped objective}
The gradient of the clipped objective \eqref{eq:Lclip} corresponds to:
\begin{equation*}
    \nabla_\theta \LClip(\theta) = \E_{x\sim\piold} \left[
      r(a|s) \hat A_{a,s} \nabla_\theta \log\pi_\theta(a|s)
        \cdot
        \mathbbm{1}\left\{r(a|s) \in (1-\epsilon, 1+\epsilon)
        \textrm{ or } \sgn(r(a|s)-1) \neq
        \sgn(\hat A_{a,s})\right\} \right].
\end{equation*}

\subsection{KL-regularized objectives}

The gradient of the forward-KL objective \eqref{eq:Lforward} is
\begin{equation*}
    \nabla_\theta \LKLorig(\theta) = \nabla_\theta L(\theta) - \beta
    \nabla_\theta \DKL(\piold \;||\;\pi_\theta)
    = \E_{x\sim\piold}\left[
        \left(\frac{\pi_\theta}{\piold} A + \beta \right) \nabla_\theta\log\pi_\theta \right]
\end{equation*}
and the gradient of the reverse-KL objective \eqref{eq:Lreverse} is
\begin{equation*}
    \nabla_\theta \LKLflip(\theta) = \nabla_\theta L(\theta) - \beta
    \nabla_\theta \DKL(\pi_\theta \;||\;\piold)
    = \E_{x\sim\piold}\left[
        \frac{\pi_\theta}{\piold} \left(A - \beta \log(\frac{\pi_\theta}{\piold})\right) \nabla_\theta\log\pi_\theta \right].
\end{equation*}

To derive these expressions we first look at the KL penalty term. For forward KL we get
\begin{align*}
    \nabla_\theta \DKL(\piold\;||\;\pi_\theta) &= \nabla_\theta \int_x
    \piold(x) \log(\frac{\piold(x)}{\pi_\theta(x)})
    = \int_x\nabla_\theta \left(\piold(x) \log(\frac{\piold(x)}{\pi_\theta(x)})\right) \\
    & = -\int_x \piold(x) \nabla_\theta(\log(\pi_\theta(x))) \\
    & = -\E_{x\sim\piold} \left[\nabla_\theta(\log \pi_\theta) \right]
\end{align*}
and for reverse-KL
\begin{align*}
    \nabla_{\theta}\DKL(\pi_\theta\;||\;\piold) & =  \nabla_\theta \int_x
    \pi_\theta(x) \log(\frac{\pi_\theta(x)}{\piold(x)}) 
     =  \int_x \nabla_\theta \left( \pi_\theta(x)
    \log(\frac{\pi_\theta(x)}{\piold(x)}) \right) \\
    & = \int_x \nabla_\theta (\pi_\theta(x)) \log(\frac{\pi_\theta(x)}{\piold(x)})
    + \int_x \pi_\theta(x) \nabla_\theta \log(\pi_\theta(x)) \\
    & = \E_{x \sim \piold} \left[ \frac{\pi_\theta}{\piold}
    \log(\frac{\pi_\theta}{\piold}) \nabla_\theta (\log \pi_\theta) \right],
\end{align*}
where we used
$\int_x \pi_\theta(x) \nabla_\theta \log(\pi_\theta(x)) = \int_x \nabla_\theta
(\pi_\theta(x)) = \nabla_\theta \int_x\pi_\theta(x) = \nabla_\theta (1) = 0$.

When $\frac{\piold}{\pi_\theta} \approx 1$, by Taylor expansion
\begin{equation*}
     \frac{\pi_\theta}{\piold} \log\left(\frac{\pi_\theta}{\piold}\right) \approx 
     \frac{\pi_\theta}{\piold} - 1 + \frac{1}{2}\left(\frac{\pi_\theta}{\piold} -
     1\right)^2,
\end{equation*}
so using the identity $\int_x \pi_\theta(x) \nabla_\theta \log(\pi_\theta(x)) = 0$ again,
\begin{align*}
    \nabla_\theta \DKL(\pi_\theta\;||\;\piold) & \approx
    -\E_{x\sim\piold} \left[ \nabla_\theta(\log \pi_\theta) \right]
    + \frac{1}{2}\E_{x\sim\piold} \left[\left|\frac{\pi_\theta}{\piold} - 1\right|^2
    \nabla_\theta(\log \pi_\theta)\right]
    \\ & =
    \nabla_\theta \DKL(\piold\;||\;\pi_\theta)
    + \frac{1}{2}\E_{x\sim\piold} \left[\left|\frac{\pi_\theta}{\piold} - 1\right|^2 \nabla_\theta(\log \pi_\theta) \right]
\end{align*}

Or equivalently,
the difference between the two KL penalties is
\begin{equation*}
    \nabla_\theta \DKL(\pi_\theta\;||\;\piold) - \nabla_\theta
    \DKL(\piold\;||\;\pi_\theta) \approx
    \frac{1}{2}\E_{x\sim\piold} \left[|\frac{\pi_\theta}{\piold} - 1|^2 \nabla_\theta(\log \pi_\theta) \right]
\end{equation*}

This difference in gradients is only large if $\frac{\pi_\theta}{\piold}$ is far away
from 1 where the score function $\nabla_\theta(\log\pi_\theta)$ has large magnitude.
We empirically evaluate the correlation of $\DKL(\pi_\theta\;||\;\piold)$ and $ \DKL(\piold\;||\;\pi_\theta) $ in Figure \ref{fig:kl_correlation}

\subsection{Weighting of examples}
We state in Table~\ref{tab:gradient_weighting} the weighting of examples in the gradient calculation for each of the four surrogate objectives.
    Since $\E_{x\sim\piold}[\frac{\pi_\theta}{\piold}\nabla_\theta\log\pi_\theta] = 0$, the weightings are up to $\frac{\textrm{constant}}{A(x)}$. We chose the constants such that the weighting is $1$ when $\pi_\theta = \piold$.

\begin{table}[h!]
    \centering
    \caption{Gradients interpreted as weighting of examples.}
    \begin{tabular}{l|c}
        \toprule
         & Weighting of examples\\
         \midrule
        $\nabla_\theta L$ & 1 \\ 
        $\nabla_\theta \LKLflip$ & $1 - \frac{\beta}{A(x)}\log\left(\frac{\pi_\theta(x)}{\piold(x)}\right) = 1 + \frac{\beta}{A(x)}\log\left(\frac{\piold(x)}{\pi_\theta(x)}\right)$ \\
        $\nabla_\theta \LKLorig$ & $1 + \frac{\beta}{A(x)}\left(\frac{\piold(x)}{\pi_\theta(x)}-1\right)$ \\
        $\nabla_\theta \LClip$ & $\mathbbm{1}\left\{\frac{\pi_\theta(x)}{\piold(x)} \in (1-\epsilon, 1+\epsilon)
                                 \textrm{ or } \sgn\left(\frac{\pi_\theta(x)}{\piold(x)}-1\right) \neq \sgn(A(x))\right\}$\\
        \bottomrule
    \end{tabular}
    \label{tab:gradient_weighting}
\end{table}

\newpage

\section{Experimental setups}

\subsection{Constructed environments}
In all constructed environments, we use the same training configuration: 
In each PPO iteration, we sample 512 timesteps (16 batches, 32 in each minibatch) and run 10 epochs with learning rate 0.1.
We chose a relatively high learning rate in order to illustrate the failure behavior in a small number PPO iterations for illustrative purposes.
Similar failure examples also happen with varying batch sizes and learning rates. 

\textit{Single-peak 1D environment (Figure~\ref{fig:cts_bandit_failure}).} The reward is 1.0 on the interval $(-1.0, -0.8)$, and zero otherwise, with Gaussian noise of standard deviation 0.1.
For Beta policy and discretized policy, we set the action bounds to be $[-1.5, 1.5]$. The discrete policy discretizes the action space uniformly in 0.1 increments.

\textit{Double-peak 1D environment (Figure~\ref{fig:double_peak_failure}).}
The reward is given by
\begin{equation*}
    r(a) = 1.1 \times \exp\{-1.2\times (a+2)^2\} + 0.9 \times \exp\{-0.9\times (a-1)^2\},
\end{equation*}
with Gaussian noise of standard deviation 0.1. The action bounds are $[-5,5]$ for Beta policy.

\textit{Discrete grid environment  (Figure~\ref{fig:disc_bandit_failure}).}
We evaluate PPO on discrete grid environments of different sizes, varying the number of actions from 10 to 100 in increments of 10. In an environment with $n$ actions, the rewards are zero on $n/2$ actions, 0.5 on $(n-1)/2$ actions, and 1 on a single optimal action. We then also add Gaussian noise of scale 0.1 to rewards. This discrete environment satisfies the assumptions in Theorem~\ref{thm:ppo_bounds} for convergence guarantees.
When evaluting the probability of converging to optimal action, we repeat 20 runs for each settin, and define `converging to optimal action` as $\geq 95\%$ probability on the optimal action in the policy after 50 iterations. When inspecting the learned policies, we find a bimodal behavior that after 50 iterations the policy has either $\geq 95\%$ probability or $\leq 5\%$ probability on the optimal action. 

\subsection{MuJoCo experiments}

\begin{table}
\centering
\begin{tabular}{lrrr}
\toprule
Environment            &  Observation Dimension & Action Dimension  &  Action Range    (per dim)                          \\
\midrule
Walker2d               &      17   &       6 &                           [-1, 1] \\
Humanoid               &      376   &    17 &                       [-0.4, 0.4] \\
Swimmer                &      8     &     2 &                           [-1, 1] \\
Hopper                 &      11     &     3 &                           [-1, 1] \\
HalfCheetah            &      17     &     6 &                           [-1, 1] \\
InvertedPendulum       &      4     &     1 &                           [-3, 3] \\
Reacher                &      11     &     2 &                           [-1, 1] \\
InvertedDoublePendulum &      11    &     1 &                           [-1, 1] \\
\bottomrule
\end{tabular}
\caption{MuJoCo environment description.}
\label{tab:mujoco_envs}
\end{table}

\begin{table}
\centering
\begin{tabular}{lrr}
\toprule
Hyperparameter & Value & Range Considered \\
\midrule
Total timesteps & 1M & N/A \\
Timesteps per iteration (horizon) & 2000 & 1000 - 4000 \\
Discount factor ($\gamma$) & 0.99 & N/A \\
GAE discount ($\lambda$) & 0.95 & N/A \\
Minibatch size & 64 & 32 - 512 \\
Clipping Threshold ($\epsilon$) & 0.2 & 0.0 - 0.5 \\ 
KL penalty coeff ($\beta$) & 3.0 & 0.1 - 10.0 \\ 
Policy \# epochs & 10 & 1 - 40 \\
    Policy LR & \hspace{4em}  $3 \times 10^{-4}$ &  \hspace{4em}   $1\times 10^{-5}$ - $5 \times 10^{-4}$ \\
Policy network hidden layers & $[64, 64]$ & N/A \\ 
Value \# epochs & 10 & N/A \\
Value LR & $2 \times 10^{-5}$ & N/A \\
Value network hidden layers & $[64, 64]$ & N/A \\ 
\bottomrule
\end{tabular}
\caption{PPO hyperparameters used for MuJoCo tasks. The clipping threshold and the KL penalty coefficient are exclusive, i.e., in the clipped objective version, the clipping threshold is 0.2 and the KL penalty coefficient is 0, while in the KL-regularized version the clipped threshold is $1\times 10^8$ and the KL penalty coefficient is 3.}
\label{tab:hyperparams}
\end{table}

We choose hyperparameters based on the default hyperparameter values in the original PPO paper~\cite{schulman2017proximal} and a recent PPO ablation study~\cite{engstrom2020implementation}. For important hyperparameters related to policy optimization (clipping threshold, KL penalty coefficient, number of epochs per PPO iteration, minibatch size, etc.) we consider a range of hyperparameters, and find that the default values work well. See Table~\ref{tab:hyperparams} for the range considered and the final values used.

Following~\cite{pmlr-v70-chou17a}, we parameterize the $\alpha$ and $\beta$ parameters in Beta distribution by softplus, with an added constant 1 to ensure $\alpha, \beta \geq 1$.

We evaluate on eight MuJoCo tasks in OpenAI gym as listed in Table~\ref{tab:mujoco_envs}. For each task, we report the mean cumulative episode reward and the 95\% confidence interval over 10 runs in all the figures.

\newpage
\section{Supplementary figures}\label{sec:supp_figs}

\begin{figure}[h]
    \centering
    \includegraphics[width=\linewidth]{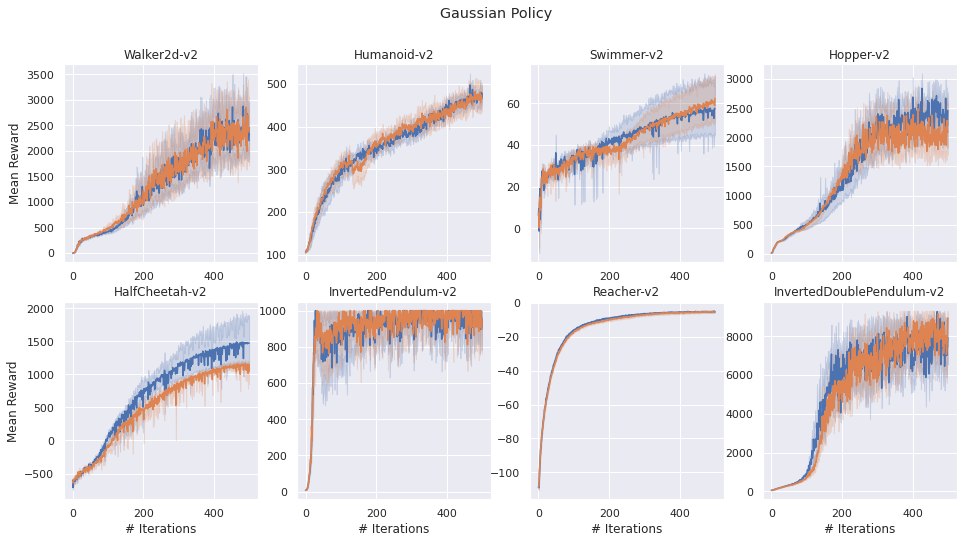}
    \includegraphics[width=\linewidth]{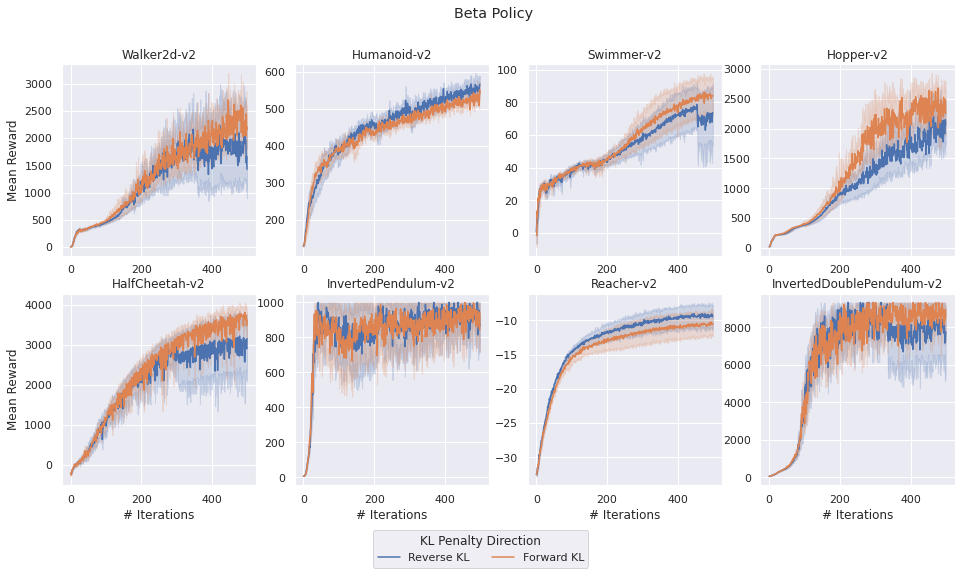}
    \caption{KL direction in KL penalty does not change PPO performance as measured by mean episode rewards on most MuJoCo locomotion tasks, whether using Gaussian policy or Beta policy.}
    \label{fig:kl_direction_comp}
\end{figure}

\begin{figure}[t]
    \centering
    \includegraphics[width=\textwidth]{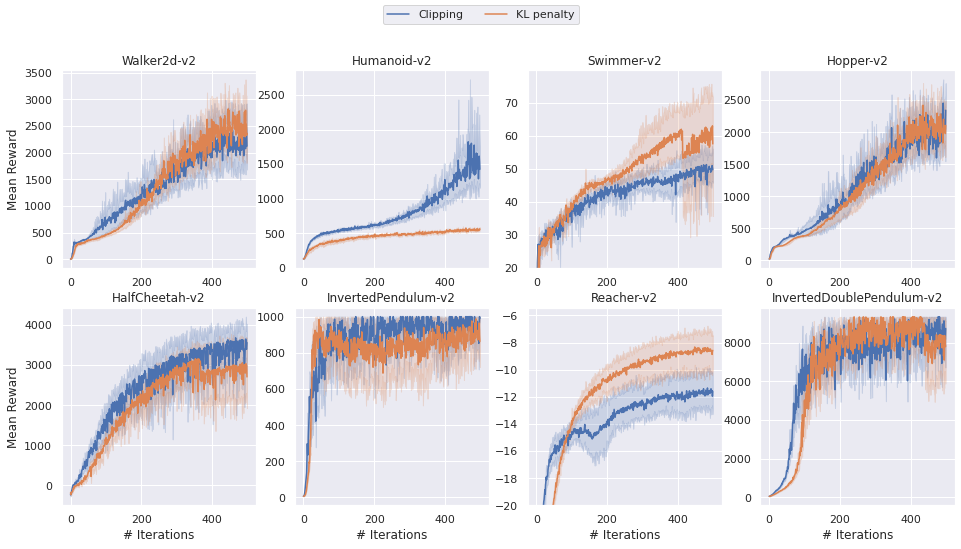}
    \caption{With Beta policy, KL-regularized PPO has similar or better performance than PPO with clipped surrogate objective on most MuJoCo tasks, with Humanoid being the exception. Humanoid performance here may not be the optimal for KL-regularized PPO, as it is higher dimensional than other tasks and may require significantly different hyperparameters tuned specifically for Humanoid.}
    \label{fig:clipping_kl_mujoco_comp_beta}
\end{figure}

\begin{figure}
    \centering
    \includegraphics[width=0.6\textwidth]{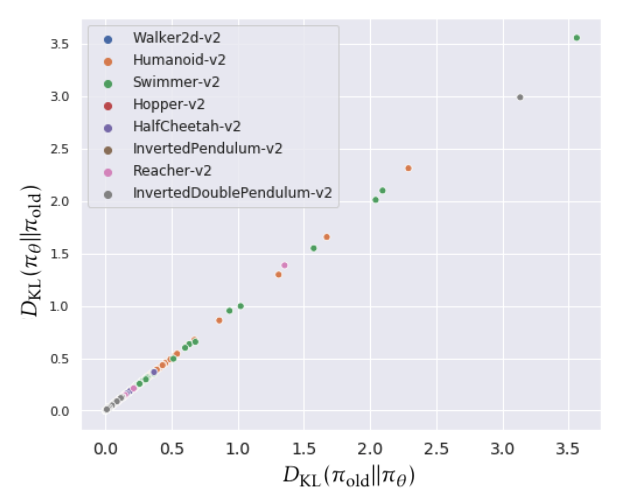}
    \caption{The empirically observed KL divergence in the two directions is nearly identical when evaluated on MuJoCo tasks.}
    \label{fig:kl_correlation}
\end{figure}

\begin{table}
    \centering
\begin{tabular}{lp{13cm}}
\toprule
\# Actions &                                              Games \\
\midrule
3         &                                    Freeway, Skiing \\
4         &                                 Atlantis, Breakout \\
6         &             Bowling, Carnival, Pong, Pooyan, Qbert \\
7         &                                            Assault \\
8         &                         Gopher, Phoenix, Tutankham \\
9         &                                    Asterix, Enduro \\
10        &                                             Amidar \\
14        &                                          Asteroids \\
18        &  Alien, Berzerk, Boxing, Centipede, Frostbite, Gravitar, Hero,
             Jamesbond, Kangaroo, Krull, Pitfall, Riverraid, Robotank,
            Seaquest, Solaris, Tennis, Venture, Zaxxon \\
\bottomrule
\end{tabular}
    \caption{Action space dimensions of all Atari environments in Open AI Gym.}
    \label{tab:atari_action_dims}
\end{table}

\clearpage

\section{Additional studies on PPO: reward scaling and advantage normalization}
\label{sec:appendix_scaling}

We include additional studies on the effects of reward scaling and advantage normalization, two common design choices in most PPO implementations.

\subsection{Reward scaling}

Previous work~\cite{engstrom2020implementation} has highlighted a surprising observation that reward scaling is crucial for PPO's success. The reward scaling scheme studied in~\cite{engstrom2020implementation} is complex and unintuitive, normalizing rewards by the standard deviation of cumulative returns. We clarify that a simple constant reward scaling scheme can replace the more complex scaling and achieve the same performance on MuJoCo tasks (Figure~\ref{fig:reward_scaling_performance}).

We hypothesize that the constant reward scaling is important for matching the scale of rewards with the scale of initialized value functions, since the rewards are compared to the value scores in advantage estimation.
To test this hypothesis, we visualize the distribution of estimated advantages during the initial 100 steps of training (Figure~\ref{fig:reward_norm_first_100}).
If this hypothesis were to be true, the advantage distributions are ``bumpy'' before reward scaling because they were dominated by the reward values, and become ``smoother'' after reward scaling as they strike more of a balance between rewards and value functione estimates. 

In the Humanoid, Walker, and Hopper tasks, reward scaling results in a huge difference in PPO performance, more than doubling the average reward. Interestingly, when comparing Figure~\ref{fig:reward_scaling_performance} and Figure~\ref{fig:reward_norm_first_100}, Hopper and Humanoid, and Walker are the same tasks where the reward scaling scheme has the largest effect on the normalized advantage distributions.
Reward scaling centers and smooths the advantage distributions in those tasks.

Further work is required to determine the affects of reward scaling and whether it can be compensated by better value function initialization.

\begin{figure}[h]
    \centering
    \includegraphics[width=\linewidth]{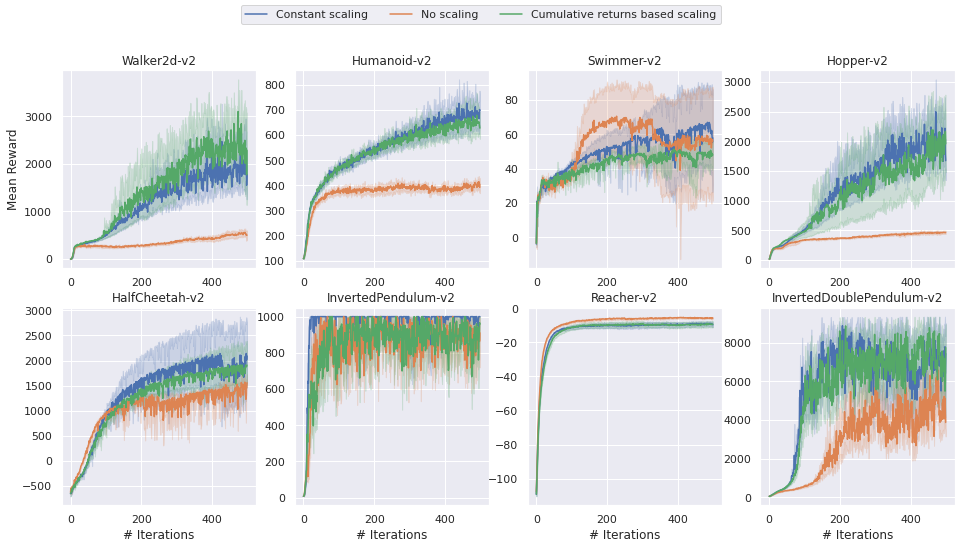}
    \caption{Reward scaling is critical for PPO's success. Simple constant scaling matches more complex cumulative returns based scaling.
    }
    \label{fig:reward_scaling_performance}
\end{figure}

\begin{figure}
    \centering
    \includegraphics[width=\linewidth]{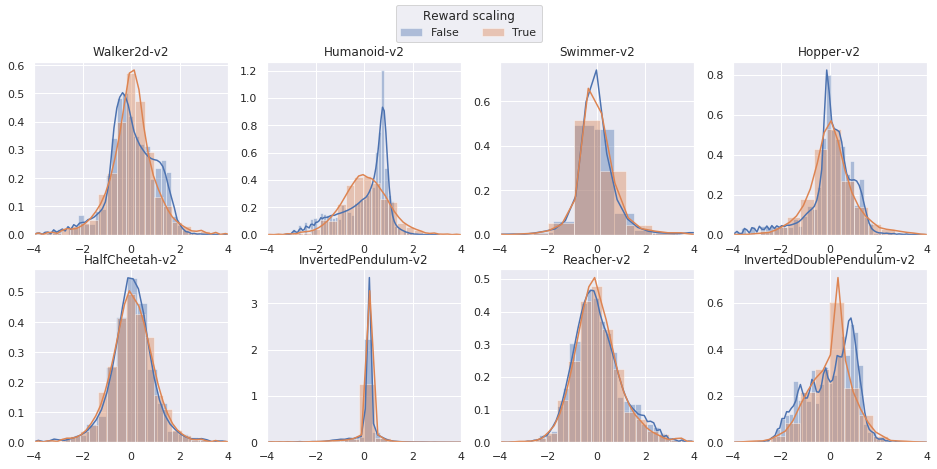}
    \caption{The distribution of normalized advantages, with and without the reward scaling scheme, in the initial phase of training. The histograms are aggregated over normalized advantage values from step 0, 10, 20, $\cdots$, 100.
    }
    \label{fig:reward_norm_first_100}
\end{figure}

\subsection{Advantage normalization}
Advantage normalization is another common heuristic in PPO. While it increases PPO performance on some MuJoCo tasks, it in fact reduces performance on some other tasks (Figure~\ref{fig:adv_norm_comp}). 

\begin{figure}[h]
\centering
\includegraphics[width=\textwidth]{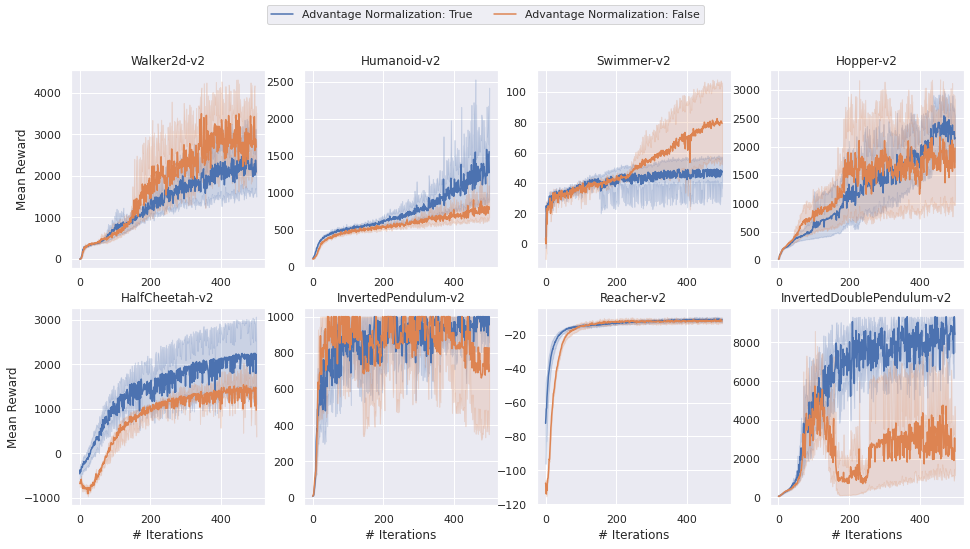}
    \caption{Comparison of PPO performance with vs without advantage normalization.}
\label{fig:adv_norm_comp}
\end{figure}

In light of the mirror descent connection (Appendix Section~\ref{sec:ppo_connections}), we can intrepret advantage normalization as dynamically adjusted mirror descent step size. Our extended analysis in Theorem~\ref{thm:ppo_bounds} accounts for advantage normalization and provides insights on how it influences PPO convergence bounds.

When inspecting the standard deviation of advantage estimates before normalization (Figure~\ref{fig:adv_std}), we notice that the standard deviation either decreases rapidly and flattens, or gradually increases, depending on the environment. This suggests that advantage normalization effectively anneals the mirror descent step size over the course of training.

It remains interesting in future work to determine more precise conditions of when advantage normalization helps or hurts.

\begin{figure}
\centering
\includegraphics[width=\textwidth]{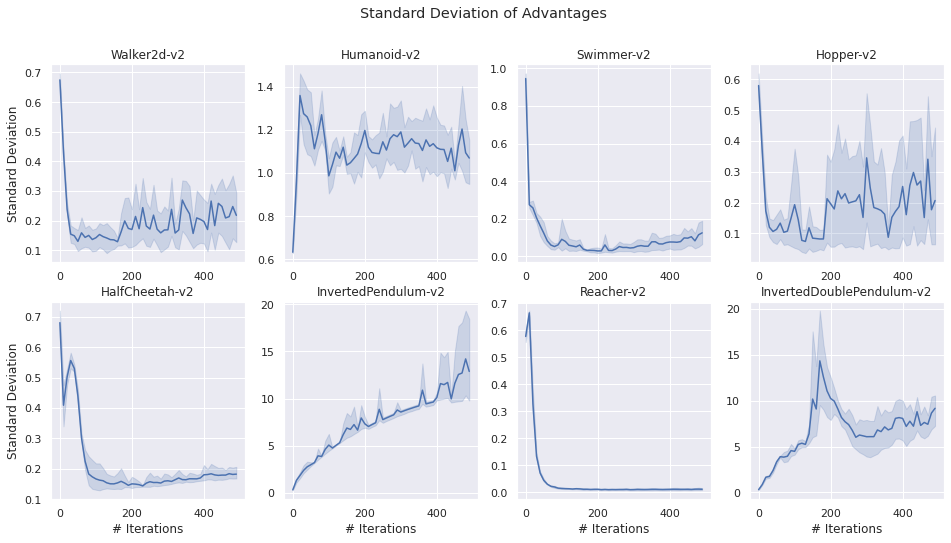}
    \caption{Standard deviation of advantage estimates. Over the course of training, the standard deviation either decreases rapidly and flattens, or gradually increases, depending on the environment.}
\label{fig:adv_std}
\end{figure}

\end{document}